\documentclass[journal]{IEEEtran}

\usepackage{times}
\usepackage{soul}
\usepackage{url}
\usepackage[hidelinks]{hyperref}
\usepackage[utf8]{inputenc}
\usepackage[small]{caption}
\usepackage{graphicx}
\usepackage{amsmath}
\usepackage{amsthm}
\usepackage{booktabs}
\usepackage{algorithm}
\usepackage{algorithmic}
\usepackage[switch]{lineno}
\usepackage{multirow}
\usepackage{amsthm} 
\newtheorem{definition}{Definition}
\newtheorem{lemma}{Lemma}
\newtheorem{theorem}{Theorem} 

\usepackage{amssymb} 
\usepackage{booktabs}
\usepackage{multirow}
\usepackage{diagbox}
\usepackage{xcolor}
\usepackage{placeins} 
\usepackage{bm}

\definecolor{OliveGreen}{rgb}{0.33, 0.42, 0.18}
\usepackage{xcolor}

\hypersetup{
    colorlinks=true,   %
    citecolor=blue,    %
    linkcolor=red,     %
    urlcolor=blue     %
}

\begin{document}

\title{Deep Multi-Manifold Transformation Based Multivariate Time Series Fault Detection}

\author{
    Hong~Liu,~\IEEEmembership{Member,~IEEE},
    Xiuxiu~Qiu,
    Yiming~Shi,
    Miao~Xu,
    Zelin~Zang,~\IEEEmembership{Member,~IEEE}, %
    and Zhen~Lei,~\IEEEmembership{Fellow,~IEEE},
    \thanks{Hong Liu is with the School of Information and Electrical Engineering, Hangzhou City University, Hangzhou 310015, China, and the Academy of Edge Intelligence, Hangzhou City University, Hangzhou 310015, China.}%
    \thanks{Xiuxiu Qiu is with the College of Information Engineering, Zhejiang University of Technology, Hangzhou 310012, China.}%
    \thanks{Yiming Shi is with the Institute of Cyber-Systems and Control, Zhejiang University, Hangzhou 310027, China.}%
    \thanks{Miao Xu is with the Centre for Artificial Intelligence and Robotics (CAIR), HKISI-CAS, Hong Kong, China.}%
    \thanks{Zhen Lei is with the Centre for Artificial Intelligence and Robotics (CAIR), HKISI-CAS, Hong Kong, China; the State Key Laboratory of Multimodal Artificial Intelligence Systems (MAIS), Institute of Automation, Chinese Academy of Sciences (CASIA), Beijing, China; and the School of Artificial Intelligence, University of Chinese Academy of Sciences (UCAS), Beijing, China (email: zhen.lei@ia.ac.cn).}%
    \thanks{Zelin Zang is with the School of Engineering, Westlake University, Hangzhou 310015, China. and Centre for Artificial Intelligence and Robotics (CAIR), HKISI-CAS, Hong Kong, China. (Corresponding author: Zelin Zang, email: zangzelin@westlake.edu.cn)}%
}

\maketitle

\begin{abstract} 
Unsupervised fault detection in multivariate time series plays a vital role in ensuring the stable operation of complex systems. Traditional methods often assume that normal data follow a single Gaussian distribution and identify anomalies as deviations from this distribution. {\color{black} However, this simplified assumption fails to capture the diversity and structural complexity of real-world time series, which can lead to misjudgments and reduced detection performance in practical applications. To address this issue, we propose a new method that combines a neighborhood-driven data augmentation strategy with a multi-manifold representation learning framework.} By incorporating information from local neighborhoods, the augmentation module can simulate contextual variations of normal data, enhancing the model's adaptability to distributional changes. In addition, we design a structure-aware feature learning approach that encourages natural clustering of similar patterns in the feature space while maintaining sufficient distinction between different operational states. Extensive experiments on several public benchmark datasets demonstrate that our method achieves superior performance in terms of both accuracy and robustness, showing strong potential for generalization and real-world deployment.
\end{abstract}


\section{Introduction}
\label{sec:intro}
{\color{black} Unsupervised fault detection~\cite{liu2025usd,huang2025graph} in multivariate time series has become increasingly vital in both academic research and practical domains, especially within industrial environments and the management of large-scale systems~\cite{wang2022multiscale,lin2022asynchronous,zhang2023integrated}.} By eliminating the need for labeled data, it enables early identification of anomalies and incipient faults~\cite{zhao2023novel,chen2021data}, which is crucial for minimizing maintenance overhead and averting potential system failures~\cite{li2020systematic}. This capability makes it an indispensable tool for real-time monitoring and predictive maintenance in scenarios where labeled samples are scarce or entirely unavailable~\cite{zhou2024label}.

\begin{figure}[t]
  \includegraphics[width=\linewidth]{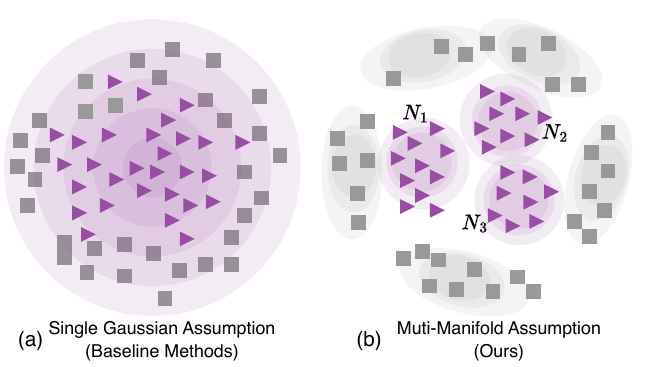} 
  \vspace{-0.2cm}
  \caption{
    {\color{black}\textbf{Motivation of DMTFD.} (a) Existing methods rely on Gaussian assumptions and fail to capture complex latent structures with multiple sub-states. (b) Our proposed DMTFD introduces multi-manifold assumption to better model state variations and improve fault detection accuracy.}}
  \label{fig_motivation}
  \vspace{-0.2cm}
\end{figure}

To address this challenge, a wide range of approaches have been explored, from classical statistical methods to modern deep learning-based models~\cite{zhang2021end,yu2023statistical}. Despite the notable progress made in recent years~\cite{li2020markov}, most existing techniques still exhibit unsatisfactory detection performance in practical applications, characterized by frequent false positives and missed anomalies (see Table~\ref{tab_my_label_roc} and Table~\ref{tab_my_label_pr}). This is largely due to the restrictive assumptions imposed during modeling, often necessitated by the absence of supervision. A prevailing strategy is to model normal behavior as a Gaussian distribution and to flag deviations as anomalies~\cite{zhou2023detecting}, as illustrated in Fig.~\ref{fig_motivation}(a). While computationally convenient, such assumptions fail to capture the intrinsic diversity and nuanced dynamics of real-world systems.

{\color{black}
In reality, both normal and anomalous behaviors often comprise multiple heterogeneous sub-patterns, each exhibiting distinct temporal and structural traits (Fig.\ref{fig_motivation}(b)). Relying on a single Gaussian model obscures these internal variations, leading to suboptimal representations and a reduced capacity to detect subtle deviations~\cite{liu2025usd}. This oversimplification not only impairs the model's ability to capture gradual state transitions but also undermines robustness in highly variable environments\cite{feng2023computation,peng2022fault}. Consequently, there is a pressing need to move beyond Gaussian-centric assumptions and adopt more expressive modeling paradigms capable of characterizing complex, multi-modal system behaviors.
}

{\color{black}
To address the limitations caused by the oversimplified Gaussian distribution assumption, we propose an unsupervised framework, termed DMTFD (Deep Multi-manifold Transformation for Fault Detection). Unlike prior methods that characterize normal states with a single Gaussian distribution~\cite{zhou2023detecting}, DMTFD adopts a multi-manifold assumption, which acknowledges the intrinsic diversity of normal operating conditions across different system states (see Fig.~\ref{fig_motivation}(b)).}

Concretely, DMTFD leverages a multi-manifold transformation module that maps multivariate time series into a latent space where heterogeneous local patterns are better disentangled and aligned. This transformation enables the model to preserve fine-grained distinctions between sub-manifolds of normal behaviors while amplifying their deviation from anomalous ones. To construct a meaningful representation space under this assumption, we introduce a neighbor-aware augmentation strategy that generates context-consistent variants of each sample~\cite{li2024genurl}, and adopt a softened similarity constraint to gently pull semantically close instances. This design alleviates the rigidity of binary contrastive signals and supports more nuanced anomaly boundary modeling~\cite{zonta2020predictive}.

Comprehensive experiments show that DMTFD significantly outperforms existing methods in both AUC and PR metrics. On benchmark datasets, DMTFD achieves consistently lower false alarm rates and higher detection accuracy. 
{\color{black}
  Furthermore, our visualization analysis supports the plausibility of the multi-Gaussian assumption and illustrates how soft contrastive learning effectively captures underlying sub-state structures.
{\textbf{(a) Multi-manifold modeling of operational states.}} We break away from the conventional single-Gaussian assumption and adopt a multi-manifold view that reflects the heterogeneity of both normal and abnormal conditions. This insight leads to a more faithful and flexible modeling of system behaviors.
{\textbf{(b) New framework for unsupervised representation.}} We propose a novel framework that integrates neighbor-based data augmentation and loss function to construct a smooth and discriminative latent space.
{\textbf{(c) Superior empirical performance across benchmarks.}} Our method achieves over 5\% improvement in detection metrics compared to existing baselines, demonstrating strong robustness and generalization in real-world industrial scenarios.
}

\section{Related Work}
\label{sec:related}
\subsection{Time Series Anomaly Detection}

In time series anomaly detection, One-Class Classification (OCC) methods—such as USAD~\cite{audibert2020usad}, and DAEMON~\cite{chen2021daemon}—typically assume access to purely normal data during training~\cite{chalapathy2019deep,hundman2018detecting}. GANF~\cite{dai2021graph} employs graph neural networks for anomaly detection in MTS. MTGFlow~\cite{zhou2023detecting} combines dynamic graphs with normalizing flows, while MTGFlow-cluster~\cite{zhou2024label} further boosts accuracy by clustering entities. {\color{black}AnomalyLLM~\cite{liu2024large} leverages large language models for knowledge distillation, using prototypical signals and synthetic anomalies to achieve a improvement on UCR datasets across 15 benchmarks. PeFAD~\cite{xu2024pefad} introduces PLM-based parameter-efficient federated learning with anomaly masking and synthetic distillation, improving performance by up to 28.74\%. Graph-MoE~\cite{huang2025graph} enhances GNN-based detection via expert mixtures and memory routers, effectively utilizing hierarchical features.
In addition, prior work USCL~\cite{liu2025usd} proposes a contrastive learning framework tailored for multivariate time series fault detection, addressing limitations of hard contrastive loss in the presence of view-level noise. Building upon this, the present work introduces a multi-manifold transformation strategy and a generalized similarity modeling approach, enabling more expressive representations and significantly improving detection performance and robustness across benchmarks.
}

\subsection{Contrastive Learning and Soft Contrastive Learning}
Soft contrastive learning~\cite{zang2023boosting,zang2022dlme} and deep manifold learning~\cite{nguyen2022deep,li2020deep} represent two advanced methodologies in the domain of machine learning, particularly in the tasks of unsupervised learning and representation learning. These techniques aim to leverage the intrinsic structure of the data to learn meaningful and discriminative features. At its core, soft contrastive learning~\cite{zang2023boosting,zang2022dlme,tang2024advancing} is an extension of the contrastive learning framework~\cite{chen2020simple,grill2020bootstrap,he2020momentum}, which aims to learn representations by bringing similar samples closer and pushing dissimilar samples apart in the representation space. However, soft contrastive learning introduces a more nuanced approach by incorporating the degrees of similarity between samples into the learning process, rather than treating similarity as a binary concept. This is achieved through the use of soft labels or continuous similarity scores, allowing the model to learn richer and more flexible representations. The softness in the approach accounts for the varying degrees of relevance or similarity among data points, making it particularly useful in tasks where the relationship between samples is not strictly binary or categorical, such as in semi-supervised learning or in scenarios with noisy labels.

\section{Preliminary}

\begin{figure*}
    \centering
    \includegraphics[width=0.99\textwidth]{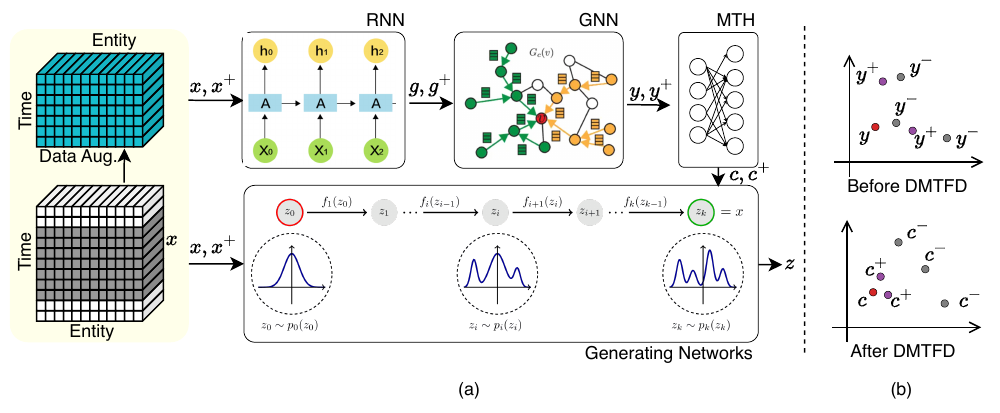}
  \vspace{-0.2cm}
    \caption{\color{black}\textbf{Overview of the proposed DMTFD framework.} (a) Temporal sequences $x, x^+$ are processed via RNN to capture temporal dynamics, followed by a GNN to model inter-entity dependencies, yielding embeddings $y$. A soft contrastive learning head maps $y$ to transformed features $c$, which are then used for density modeling via a flow-based generative network. (b) The SCL head helps cluster similar samples and separate dissimilar ones in the feature space.}
  \vspace{-0.2cm}
    \label{fig_overview}
\end{figure*}

Normalizing flow is an unsupervised density estimation approach to map the original distribution to an arbitrary target distribution by the stack of invertible affine transformations. When density estimation on original data distribution $\mathcal{X}$ is intractable, an alternative option is to estimate $z$ density on target distribution $\mathcal{Z}$. Specifically, suppose a source sample $x \in \mathcal{R}^{D} \sim \mathcal{X}$ and a target distribution sample $z \in \mathcal{R}^{D} \sim \mathcal{Z}$. Bijective invertible transformation $\mathcal{F}_{\theta}$ aims to achieve one-to-one mapping $z = f_{\theta}(x)$ from $\mathcal{X}$ to $\mathcal{Z}$. According to the change of variable formula, we can get
\begin{equation}
    P_\mathcal{X}(x) = P_\mathcal{Z}(z)\left| \det\frac{\partial{f_{\theta}}}{\partial{x}^{T}}\right|.
\end{equation}
Benefiting from the invertibility of mapping functions and tractable jacobian determinants $\left |\det\frac{\partial{f_{\theta}}}{\partial{x}^{T}}\right|$.
The objective of flow models is to achieve $\hat{z} = z$, where $\hat{z} = f_{\theta}(x)$.

Flow models are able to achieve more superior density estimation performance when additional conditions $C$ are input~\cite{ardizzone2019guided}.
Such a flow model is called conditional normalizing flow, and its corresponding mapping is derived as $z = f_{\theta}(x|C)$. Parameters $\theta$ of ${f}_{\theta}$ are updated by maximum likelihood estimation (MLE): 
\begin{equation}
    \theta^{*}=\mathop{\arg\max}\limits_{\theta}(log(P_\mathcal{Z}(f_{\theta}(x|C)) + log(\left |\det\frac{\partial{f_{\theta}}}{\partial{x}^{T}} \right|))
\end{equation}

\section{Methodology}
\label{sec:method}

\subsection{Notation and Problem Definition}
Consider a multivariate time series (MTS) data set $\mathbf{X}$. The data set encompassing $K$ entities, each with $L$ observations, and denotes as $\mathbf{X} = (\mathbf{x}_1, \mathbf{x}_2, \ldots, \mathbf{x}_K)$, where each $\mathbf{x}_k \in \mathbb{R}^{L}$. Z-score normalization is employed to standardize the time series data across different entities. A sliding window approach, with a window size of $T$ and a stride size of $S$, is utilized to sample the normalized MTS, generating training samples $\mathbf{x}^{c}$, where $c$ denotes the sampling count, and $\mathbf{x}^{c}$ represents the segment $\mathbf{x}_{[cS-T/2]:[cS+T/2]}$.

The objective of unsupervised fault detection is to identify segments $\mathbf{x}^{c}$ within $\mathbf{X}$ exhibiting anomalous behavior. This process operates under the premise that the normal behavior of $\mathbf{X}$ is known, and any significant deviation from this behavior is considered abnormal. Specifically, abnormal behavior is characterized by its occurrence in low-density regions of the normal behavior distribution, defined by a density threshold $\theta < \rho(\text{normal behavior})$. The task of unsupervised fault detection in MTS can thus be formalized as identifying segments $\mathbf{x}^{c}$ where $\rho(\mathbf{x}^{c}) < \theta$.

\begin{definition}[Supervised Fault Detection in MTS]
    Let $\mathcal{D} = \{(\mathbf{x}^c, y^c)\}_{i=1}^N$ be a labeled dataset for an MTS, where each segment $\mathbf{x}^c$ is annotated with a label $y^c \in \{0, 1\}$, indicating normal ($0$) or abnormal ($1$) behavior. Supervised fault detection aims to learn a function $f: \mathbb{R}^{L} \rightarrow \{0, 1\}$ that can accurately classify new, unseen segments $\mathbf{x}_{\text{new}}$ as normal or abnormal based on learned patterns of faults.
\end{definition}

\begin{definition}[Unsupervised Fault Detection under Gaussian Assumption]
    In the Gaussian assumption context, unsupervised fault detection in an MTS operates on the premise that the distribution of normal behavior can be modeled using a Gaussian distribution. Anomalies are identified as segments $\mathbf{x}^{c}$ that fall in regions of low probability under this Gaussian model, specifically where $\rho(\mathbf{x}^{c}) < \theta$, with $\theta$ being a predefined density threshold.
\end{definition}

\begin{definition}[Unsupervised Fault Detection under Multi-Gaussian Assumption]
    Under the multi-Gaussian assumption, unsupervised fault detection recognizes that the distribution of normal behavior in an MTS may encompass multiple modes, each fitting a Gaussian model. This assumption allows for a more nuanced detection of anomalies, which are segments $\mathbf{x}^{c}$ that do not conform to any of the modeled Gaussian distributions of normal behavior, detected through a composite density threshold criterion $\rho(\mathbf{x}^{c}) < \theta$.
\end{definition}

\subsection{Neural Network Structure of DMTFD}
The core idea of \textbf{DMTFD} is to learn precise and discriminative latent representations by integrating data augmentation with soft contrastive learning. As illustrated in Fig.~\ref{fig_overview}, the DMTFD framework jointly leverages a Recurrent Neural Network (RNN) $R(\cdot; \theta)$, a Graph Neural Network (GNN) $G(\cdot; \phi)$, and a normalizing flow model $F(\cdot; \alpha)$ to effectively model the temporal patterns, relational structures, and distributional properties of multivariate time series data for fault detection.

    {\color{black}\textbf{RNN-GNN-Flow Architecture.} To jointly capture temporal dynamics, inter-entity dependencies, and density patterns, we adopt an RNN-GNN-Flow architecture. Given the multivariate input at time window $c$, denoted as $x^c$, we first compute the temporal encoding $g^c = R(x^c; \theta)$ using a recurrent neural network parameterized by $\theta$; a graph neural network with parameters $\phi$ then produces the spatio-temporal embedding $y^c = G(g^c; \phi)$; finally, a normalizing flow with parameters $\alpha$ estimates the log-likelihood score $\ell^c = F(y^c; \alpha)$, which is used to detect anomalies—where low likelihood indicates high abnormality.
    }

    {\color{black}\textbf{Manti-manifold Transformation Head~(MTH).}  }
To enhance the discriminative quality of the learned representations, we introduce a Multi-manifold Transformation Head, implemented as a multilayer perceptron (MLP) $H(\cdot; \omega)$. This module receives the spatio-temporal embedding $y^c$ from the GNN as input and projects it into a latent feature space optimized for contrastive learning:
$
    z^c = H(y^c; \omega),
    \label{equ_fsnetH}
$
where $\omega$ denotes the learnable parameters of the MLP, and $z^c$ is the contrastive representation corresponding to time window $c$.
The goal of this module is to enforce semantic alignment between similar instances while ensuring separation among dissimilar ones. This is achieved via a soft contrastive loss, which penalizes embeddings of positive pairs (semantically similar samples) that are far apart and encourages larger distances for negative pairs (semantically dissimilar samples) that are too close. Unlike hard contrastive formulations that impose strict binary similarity, the soft version allows for graded similarity relationships, thereby providing a smoother optimization landscape and improving robustness to intra-class variability.

\subsection{Maximum Likelihood Estimation (MLE) Loss and Data Augmentation}
A key objective of the DMTFD framework is to accurately identify abnormal behaviors by modeling the underlying distribution of normal operational data. To this end, we employ \textit{Maximum Likelihood Estimation (MLE)} as the training criterion for the flow-based density estimation module. The core motivation behind this design is to encourage the model to assign high likelihoods to normal patterns while assigning low likelihoods to anomalous deviations, thus enabling precise and fine-grained fault detection.

Formally, given the contrastive representation $z^c$ for time window $c$, the normalizing flow transformation $f(\cdot; \alpha)$ maps it into a latent variable $\hat{z}^c = f(z^c; \alpha)$ that follows a known target distribution (typically standard Gaussian with mean $\mu$ and identity covariance). The MLE loss is then defined as:
\begin{equation}
    \mathcal{L}_{\text{MLE}} = \frac{1}{N} \sum_{c=1}^N \left[ -\frac{1}{2}\left\|\hat{z}^c - \mu\right\|_2^2 + \log \left|\det \left( \frac{\partial f_\alpha}{\partial z^c} \right) \right| \right],
    \label{equ_Lmle}
\end{equation}
where $N$ is the number of time windows, and $\mu$ is the mean of the target Gaussian distribution $\mathcal{Z}$. The first term encourages $\hat{z}^c$ to match the target distribution, while the second term—i.e., the log-determinant of the Jacobian—ensures that the flow transformation remains invertible and probability-preserving. By optimizing $\mathcal{L}_{\text{MLE}}$, the flow model learns to model the density landscape of normal patterns precisely. At inference time, data points with significantly low log-likelihoods under the learned distribution are flagged as anomalies, enabling accurate and interpretable fault detection across diverse temporal and relational contexts.

Data augmentation is crucial for improving the model's ability to generalize beyond the training data by artificially expanding the dataset's size and diversity~\cite{zang2023boosting_icml}. Our methodology incorporates two primary strategies: neighborhood discovery and linear interpolation.

\textbf{Neighborhood Discovery.} This technique focuses on leveraging the local structure of the data to generate new, plausible data points that conform to the existing distribution. For each data point $x^c$, we define its neighborhood $N(x^c)$ using a distance metric $d(\cdot,\cdot)$, such as the Euclidean distance. The neighborhood consists of points $x^c_\text{new}$ that meet the criterion:
\begin{equation}
    d(x^c, x^c_\text{new}) \leq \epsilon
    \label{equ_neighbor_discover}
\end{equation}
where $\epsilon$ is a predetermined threshold defining the neighborhood's radius. This method captures the local density of the data, enabling the generation of new samples within these densely populated areas and thus enhancing the dataset with variations that align with the original data distribution.

\textbf{Data Augmentation by Linear Interpolation.} Data augmentation strategy augments the dataset by creating intermediate samples between existing data points. {For two points $x^c$ and $x^c_\text{new}$, a new sample $x_\text{new}^c$ is formulated as,}
\begin{equation}
    x_\text{new}^c = \alpha x^c + (1 - \alpha) x^c_\text{new},
    \label{eq_aug_all}
\end{equation}
where $\alpha$ is a random coefficient sampled from a uniform distribution, $\alpha \sim U(0,1)$. This approach facilitates a smooth transition between data points, effectively bridging gaps in the data space and introducing a continuum of sample variations. By interpolating between points either within the same neighborhood or across different neighborhoods, we substantially enhance the dataset's diversity and coverage, providing the model with a more comprehensive set of examples for training.

\subsection{Multi-Manifold Loss with Data Augmentation}
In data-augmentation-based contrastive learning (CL), the task is framed as a binary classification problem over pairs of samples. Positive pairs, drawn from the joint distribution $(x^\text{c1}, x^{c2}) \sim P_{x^\text{c1}, x^{c2}}$, are labeled as $\mathcal{H}_{c1, c2} = 1$, whereas negative pairs, drawn from the product of marginals $(x^\text{c1}, x^{c2}) \sim P_{x^\text{c1}}P_{x^{c2}}$, are labeled as $\mathcal{H}_{ck} = 0$. { The goal of contrastive learning is to learn representations that maximize the similarity between positive pairs and minimize it between negative pairs, utilizing the InfoNCE loss.} {\color{black} The $\mathcal{L}_\text{CL}\left(x^\text{c1}, x^{c2}, \{x^\text{cn}\}_{cn=1}^{N_K}\right)=$}
\begin{equation}
    \begin{aligned}
        -\log\!\frac{\exp({{z^\text{c1}}^T\!z^\text{c2}})}{\sum_{k=1}^{N_K}\!\exp({{z^\text{c1}}^T\!z^\text{cn}})} = \!-\!\log\!\frac{\exp(S(z^\text{c1},\!z^\text{c2}))}{\sum_{k=1}^{N_K}\exp(S(z^\text{c1},\!z^\text{cn}))},
    \end{aligned}
    \label{eq_CL}
\end{equation}
where $(x^\text{c1}, x^\text{c2})$ constitutes a positive pair and $(x^\text{c1}, x^\text{cn})$ a negative pair, with $z^\text{c1}, z^\text{c2}, z^\text{cn}$ being the embeddings of $x^\text{c1}, x^\text{c2}, x^\text{cn}$ respectively, and $N_K$ representing the number of negative pairs. The similarity function $S(z^\text{c1},z^\text{c2})$ is typically defined using cosine similarity. This method effectively enhances the model's discriminative power by distinguishing between closely related (positive) and less related (negative) samples within the augmented data space.

The conventional contrastive learning (CCL) loss is structured around a single positive sample contrasted against multiple negatives. To refine this, we have restructured the CCL loss into a more nuanced form that utilizes labels for positive and negative samples, denoted by $\mathcal{H}_\text{c1,c2}$. Detailed explanations of this transformation from Eq.~(\ref{eq_CL}) to Eq.~(\ref{eq_nce2}) are available in supplementary material.A. {\color{black} The $\mathcal{L}_\text{CCL}(x^\text{c1}, \{x^\text{c2}\}_{\text{c2}=1}^{N_K}) =$  }
\begin{equation}
    \begin{aligned}
        -\sum_{j=1} \{
        \mathcal{H}_\text{c1,c2} \log Q_\text{c1,c2} + (1 - \mathcal{H}_\text{c1,c2}) \log \dot{Q}_\text{c1,c2}
        \},
        \label{eq_nce2}
    \end{aligned}
\end{equation}
where $\mathcal{H}_\text{c1,c2}$ indicates if samples $c1$ and $c2$ have been augmented from the same source. $\mathcal{H}_\text{c1,c2} = 1$ signifies a positive pair $(x^\text{c1}, x^\text{c2})$, and $\mathcal{H}_{\text{c1},\text{c2}} = 0$ indicates a negative pair. The term $Q_{\text{c1},\text{c2}} = \exp(S(z^\text{c1}, z^\text{c2}))$ represents the density ratio, as defined and computed by the backbone network.
To enhance robustness against view-level noise introduced by data augmentation, we propose the {\color{black} \textit{Multi-Manifold Loss (MML)}}. Unlike conventional contrastive losses that treat pairwise labels as binary constants, MML introduces soft similarity-based weights, enabling the model to down-weight uncertain or noisy samples during training.

Given an anchor sample $x^{c_1}$ and a set of associated samples $\{x^{c_2}\}_{c_2=1}^{N_K}$, the MML loss is defined as, $\mathcal{L}_\text{MML}(x^{c_1}, \{x^{c_2}\})=$
\begin{equation}
    -\sum_{c_2=1}^{N_K} \left[ P_{c_1,c_2} \log Q_{c_1,c_2} + (1 - P_{c_1,c_2}) \log (1 - Q_{c_1,c_2}) \right],
    \label{equ_Lmml}
\end{equation}
where $P_{c_1,c_2}$ is the soft label (weight) reflecting the similarity in input space, and $Q_{c_1,c_2}$ denotes the similarity in the contrastive latent space. They are defined as,
\begin{equation}
    \begin{aligned}
        P_{c_1,c_2} & =
        \begin{cases}
            e^{\alpha} \cdot \kappa(y^{c_1}, y^{c_2}) & \text{if } \mathcal{H}_{c_1,c_2} = 1 \\
            \kappa(y^{c_1}, y^{c_2})                  & \text{otherwise}
        \end{cases}, \quad \\
        Q_{c_1,c_2} & = \kappa(z^{c_1}, z^{c_2}),
    \end{aligned}
\end{equation}
where $\mathcal{H}_{c_1,c_2} \in \{0, 1\}$ indicates whether $x^{c_2}$ is a positive pair of $x^{c_1}$, $\alpha \in [0, 1]$ is a confidence prior that emphasizes positive pairs, and $\kappa(\cdot,\cdot)$ is a similarity kernel.

    {\color{black}
        To improve robustness against augmentation noise and sample-level variability, we adopt a generalized Gaussian kernel $\kappa^\beta(\cdot, \cdot)$ as the similarity function in MML. Compared to the standard Gaussian kernel, the generalized Gaussian provides a tunable shape parameter that allows for heavier tails when needed, enabling the model to assign meaningful similarity scores even to moderately distant pairs. The kernel is defined as:
        \begin{equation}
            \kappa^\beta(a, b) = \exp\left( -\left( \frac{\|a - b\|_2}{\sigma} \right)^\beta \right),
        \end{equation}
        where $\beta > 0$ controls the shape of the decay (with $\beta=2$ reducing to a Gaussian kernel, and $\beta < 2$ producing heavier tails), and $\sigma$ is a scale parameter. This formulation makes the loss function more tolerant to view-level noise and non-uniform sample structures, leading to smoother gradients and improved generalization under multi-manifold distributions. Unlike USCL~\cite{liu2025usd}, which relies on a fixed Gaussian assumption and sharp contrastive margins, our generalized kernel formulation offers adaptive flexibility across manifold structures and improves tolerance to cross-view variation. Detailed comparisons and analysis are provided in Supplementary Material A. }

\subsection{Optimization Objectives}
The entire DMTFD framework, encompassing the RNN, GNN, and FLOW models, is optimized jointly through Maximum Likelihood Estimation (MLE) to ensure effective anomaly detection. {The joint optimization process is formulated as follows,}
\begin{equation}
    L_{\text{DMTFD}} = \mathcal{L}_{\text{MLE}} + \mathcal{L}_{\text{MML}},
    \label{equ_losstp}
\end{equation}
where $\mathcal{L}_{\text{MLE}}$ is the MLE loss, and $\mathcal{L}_{\text{MML}}$ is the MML loss. The MLE loss is designed to maximize the likelihood of observing the transformed data points $\hat{z}_i$ under the model, while the MML loss is designed to enhance the model's sensitivity to subtle differences between normal and abnormal patterns. By embedding these mathematical formulations into each step of the DMTFD framework, we ensure a rigorous approach to modeling and detecting anomalies in multivariate time series data, setting a new standard for unsupervised fault detection in complex systems.

\section{Experiments}
\label{sec:experiments}
\subsection{Dataset Information and Implementation Details.}
{\color{black} \textbf{Datasets information.} We evaluate DMTFD on six widely used multivariate time series (MTS) fault detection datasets (Table~\ref{tab_dataset_setting}). \textbf{SWaT}~\cite{goh2016dataset} and \textbf{WADI}~\cite{ahmed2017wadi} are collected from water treatment and distribution testbeds, simulating industrial cyber-attacks with labeled anomalies. \textbf{PSM}~\cite{abdulaal2021practical} contains server metrics from eBay, used to detect system performance anomalies. \textbf{MSL}~\cite{hundman2018detecting} provides telemetry from NASA's Curiosity rover for space mission anomaly detection. \textbf{SMD}~\cite{su2019robust} is from a large-scale server farm and captures system behavior across multiple machines; we report average results and ensure test sets include anomalies via fixed random seeds. SMAP is a \textbf{NASA} dataset~\cite{hundman2018detecting} that contains data from a satellite's attitude control system, used for detecting anomalies in space missions.
All datasets are standard in one-class classification (OCC) for time series anomaly detection.
}

\textbf{Datasets split and preprocessing.}
{In our experimental setup, we adhere to the dataset configurations 
used in the GANF study~\cite{dai2021graph}. MTGFlow~\cite{zhou2023detecting} and MTGFlow\_Cluster~\cite{zhou2024label}.} Specifically, for the SWaT dataset, we partition the original testing data into 60\% for training, 20\% for validation, and 20\% for testing. For the other datasets, the training partition comprises 60\% of the data, while the test partition contains the remaining 40\%. The training data is used to train the model, while the validation data is used to tune hyperparameters and the test data is used to evaluate the model's performance. The datasets are preprocessed to remove missing values and normalize the data to a range of [0, 1]. The data is then divided into fixed-length sequences, with a window size of 60 and a stride of 10. The window size determines the number of time steps in each sequence, while the stride determines the step size between each sequence. The data is then fed into the model for training and evaluation.

\begin{table*}
    \centering
    \small
    \caption{\textbf{Fault detection performance of Area Under the Receiver Operating Characteristic~(AUROC) on five public datasets.} { We compare the performance of different methods using a consistent window size (DMTFD) with our optimal results obtained through varying window sizes (DMTFD*) to showcase the maximum potential. The best results is in \textbf{bold}.}}

    \resizebox{\textwidth}{!}{
    \begin{tabular}{l|c||cccccc|c|c}
        \toprule
        \multirow{2}{*}{Methods}                                                               & \multirow{2}{*}{Journal-Year} & \multicolumn{6}{c|}{Area Under the Receiver Operating Characteristic Curve (AUROC)} & \multirow{2}{*}{Average}       & \multirow{2}{*}{Rank}                                                                                                                                                         \\
                                                                                               &                               & SWaT                                                                                & WADI                           & PSM                            & MSL                            & SMD                            & \color{black} SMAP                        &                  &               \\
        \midrule
        DROCC \cite{goyal2020drocc}                                                            & ICML, 2020                    & 72.6(\textpm{3.8})                                                                  & 75.6(\textpm{1.6})             & 74.3(\textpm{2.0})             & 53.4(\textpm{1.6})             & 76.7(\textpm{8.7})             & \color{black} 70.1(\textpm{3.5})          & 70.4             & 9             \\
        DeepSAD \cite{ruff2019deep}                                                            & ICLR, 2020                    & 75.4(\textpm{2.4})                                                                  & 85.4(\textpm{2.7})             & 73.2(\textpm{3.3})             & 61.6(\textpm{0.6})             & 85.9(\textpm{11.1})            & \color{black} 72.8(\textpm{4.1})          & 77.4             & 8             \\
        USAD \cite{audibert2020usad}                                                           & KDD, 2020                     & 78.8(\textpm{1.0})                                                                  & 86.1(\textpm{0.9})             & 78.0(\textpm{0.2})             & 57.0(\textpm{0.1})             & 86.9(\textpm{11.7})            & \color{black} 74.5(\textpm{2.9})          & 76.9             & 7             \\
        GANF \cite{dai2021graph}                                                               & ICLR, 2022                    & 79.8(\textpm{0.7})                                                                  & 90.3(\textpm{1.0})             & 81.8(\textpm{1.5})             & 64.5(\textpm{1.9})             & 89.2(\textpm{7.8})             & \color{black} 77.6(\textpm{3.2})          & 80.5             & 6             \\
        \color{black} Autoformer \cite{wu2022autoformerdecompositiontransformersautocorrelation} & \color{black} NeurIPS, 2022     & \color{black} 81.2(\textpm{1.1})                                                      & \color{black} 89.5(\textpm{1.3}) & \color{black} 83.7(\textpm{1.8}) & \color{black} 66.0(\textpm{1.5}) & \color{black} 89.8(\textpm{7.1}) & \color{black} 78.3(\textpm{3.1})          & \color{black} 81.4 & \color{black} 5 \\
        MTGFlow \cite{zhou2023detecting}                                                       & AAAI, 2023                    & 84.8(\textpm{1.5})                                                                  & 91.9(\textpm{1.1})             & 85.7(\textpm{1.5})             & 67.2(\textpm{1.7})             & 91.3(\textpm{7.6})             & \color{black} 80.2(\textpm{2.8})          & 83.5             & 3             \\
        MTGFlow\_C \cite{zhou2024label}                                                        & TKDE, 2024                    & 83.1(\textpm{1.3})                                                                  & 91.8(\textpm{0.4})             & 87.1(\textpm{2.4})             & 68.2(\textpm{2.6})             & 91.6(\textpm{7.3})             & \color{black} 81.1(\textpm{2.7})          & 84.8             & 2             \\
        \color{black} USD \cite{liu2025usd}                                                     & \color{black} ICASSP, 2025      & \color{black} 82.7(\textpm{1.0})                                                      & \color{black} 90.5(\textpm{1.2}) & \color{black} 84.5(\textpm{1.3}) & \color{black} 66.8(\textpm{2.0}) & \color{black} 90.5(\textpm{6.8}) & \color{black} 79.5(\textpm{3.2})          & \color{black} 82.4 & \color{black} 4 \\
        \midrule
        DMTFD                                                                                  & OURS                          & 90.5(\textpm{0.9})                                                                  & 94.3(\textpm{0.4})             & \textbf{89.2(\textpm{2.0})}    & 75.0(\textpm{2.2})             & \textbf{95.0(\textpm{6.1})}    & \color{black} 85.7(\textpm{2.5})          & 88.3             & 1             \\
        DMTFD*                                                                                 & OURS                          & \textbf{92.5(\textpm{1.2})}                                                         & \textbf{95.2(\textpm{1.1})}    & \textbf{89.2(\textpm{2.1})}    & \textbf{76.5(\textpm{1.0})}    & \textbf{95.0(\textpm{6.4})}    & \color{black} \textbf{87.4(\textpm{2.1})} & \textbf{89.3}    & 1*            \\
        \bottomrule
    \end{tabular}
    }
    \label{tab_my_label_roc}
\end{table*}

\begin{table*}
    \centering
    \small
    \caption{\textbf{Anomaly detection performance of Precision-Recall(PR) on five public datasets}. { We compare the performance of different methods using a consistent window size (DMTFD) with our optimal results obtained through varying window sizes (DMTFD*) to showcase the maximum potential. The best results is in \textbf{bold}.}}
    \resizebox{\textwidth}{!}{
    \begin{tabular}{l|c||cccccc|c|c}
        \toprule
        \multirow{2}{*}{Methods}                                                               & \multirow{2}{*}{Journal-Year} & \multicolumn{6}{c|}{Area Under the Precision-Recall Curve (AUPRC)} & \multirow{2}{*}{Average}        & \multirow{2}{*}{Rank}                                                                                                                                                              \\
                                                                                               &                               & SWaT                                                               & WADI                            & PSM                             & MSL                             & SMD                              & \color{black} SMAP                         &                  &               \\
        \midrule
        DROCC \cite{goyal2020drocc}                                                            & ICML, 2020                    & 26.4 (\textpm{9.8})                                                & 16.7 (\textpm{7.1})             & 60.7 (\textpm{11.4})            & 13.2 (\textpm{0.9})             & 20.7 (\textpm{15.6})             & \color{black} 31.2 (\textpm{5.3})          & 28.2             & 9             \\
        DeepSAD \cite{ruff2019deep}                                                            & ICLR, 2020                    & 45.7 (\textpm{12.3})                                               & 23.1 (\textpm{6.3})             & 66.7 (\textpm{10.8})            & 26.3 (\textpm{1.7})             & 61.8 (\textpm{21.4})             & \color{black} 45.3 (\textpm{8.4})          & 44.8             & 7             \\
        USAD \cite{audibert2020usad}                                                           & KDD, 2020                     & 18.8 (\textpm{0.6})                                                & 19.8 (\textpm{0.5})             & 57.9 (\textpm{3.6})             & 31.3 (\textpm{0.0})             & 68.4 (\textpm{19.9})             & \color{black} 42.5 (\textpm{9.5})          & 39.8             & 8             \\
        GANF \cite{dai2021graph}                                                               & ICLR, 2022                    & 21.6 (\textpm{1.8})                                                & 39.0 (\textpm{3.1})             & 73.8 (\textpm{4.7})             & 31.1 (\textpm{0.2})             & 64.4 (\textpm{21.4})             & \color{black} 50.7 (\textpm{4.3})          & 46.8             & 6             \\
        \color{black} Autoformer \cite{wu2022autoformerdecompositiontransformersautocorrelation} & \color{black} NeurIPS, 2022     & \color{black} 35.1 (\textpm{7.2})                                    & \color{black} 39.3 (\textpm{3.9}) & \color{black} 71.4 (\textpm{6.0}) & \color{black} 29.5 (\textpm{2.5}) & \color{black} 60.2 (\textpm{18.1}) & \color{black} 48.2 (\textpm{8.6})          & \color{black} 47.3 & \color{black} 5 \\
        MTGFlow \cite{zhou2023detecting}                                                       & AAAI, 2023                    & 38.6 (\textpm{6.1})                                                & 42.2 (\textpm{4.9})             & 76.2 (\textpm{4.8})             & 31.1 (\textpm{2.6})             & 64.2 (\textpm{21.5})             & \color{black} 52.6 (\textpm{6.3})          & 50.8             & 3             \\
        MTGFlow\_C \cite{zhou2024label}                                                        & TKDE, 2024                    & 41.2 (\textpm{4.8})                                                & 42.4 (\textpm{5.2})             & 76.1 (\textpm{4.9})             & 31.2 (\textpm{2.8})             & 64.5 (\textpm{20.1})             & \color{black} 53.1 (\textpm{8.9})         & 51.4             & 2             \\
        \color{black} USCL \cite{liu2025usd}                                                     & \color{black} ICASSP, 2025      & \color{black} 37.4 (\textpm{5.9})                                    & \color{black} 40.5 (\textpm{4.1}) & \color{black} 74.2 (\textpm{5.6}) & \color{black} 30.7 (\textpm{2.4}) & \color{black} 63.3 (\textpm{18.4}) & \color{black} 51.8 (\textpm{6.1})          & \color{black} 49.6 & \color{black} 4 \\
        \midrule
        DMTFD                                                                                  & OURS                          & 53.6 (\textpm{10.1})                                               & 45.7 (\textpm{4.1})             & 83.2 (\textpm{4.9})             & 32.1 (\textpm{3.5})             & 73.8 (\textpm{16.7})             & \color{black} 60.2 (\textpm{7.2})          & 58.1             & 1             \\
        DMTFD*                                                                                 & OURS                          & \textbf{59.3} (\textpm{6.5})                                       & \textbf{60.5} (\textpm{3.2})    & \textbf{83.2} (\textpm{3.9})    & \textbf{46.2} (\textpm{1.8})    & \textbf{73.8} (\textpm{13.7})    & \color{black} \textbf{65.3} (\textpm{8.9}) & \textbf{64.7}    & \textbf{1*}   \\
        \bottomrule
    \end{tabular}
    }
    \label{tab_my_label_pr}
\end{table*}

\begin{figure*}[t]
    \centering
    \includegraphics[width=0.95\textwidth]{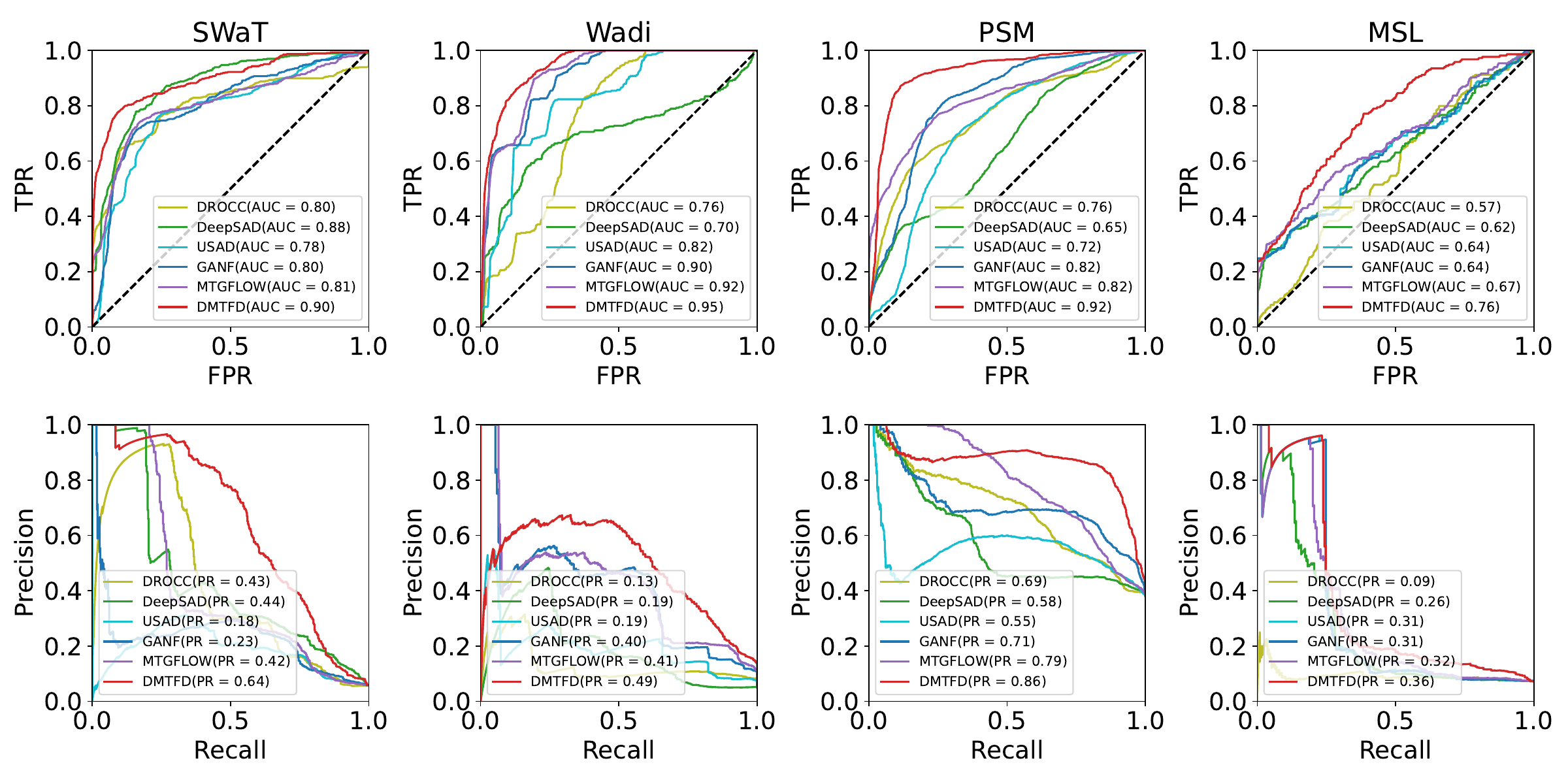}
    \vspace{-0.2cm}
    \caption{\textbf{ROC and PR plots of four datasets.} The ROC and PR curves of the SWaT, WADI, PSM, and MSL datasets are shown. The ROC curves illustrate the trade-off between true positive rate and false positive rate, while the PR curves show the trade-off between precision and recall. The performance of our method is compared against state-of-the-art methods, demonstrating the effectiveness of our approach.}
    \label{fig_ROC_PR}
\end{figure*}

\textbf{Implementation details.}
For all datasets, we set the window size to 60-80 and the stride size to 10-30.  All experiments were run for 400-600 epochs and executed using PyTorch 2.2.1 on an NVIDIA RTX 3090 24GB GPU. Additional specific parameters can be found in Supplementary Material.

\subsection{Evaluation Metric and Baselines Methods.}
Following prior work, DMTFD performs window-level anomaly detection, where a window is labeled anomalous if it contains any anomalous point. We evaluate performance using two metrics: \textbf{AUROC}, which measures overall discriminative ability across thresholds, and \textbf{AUPRC}, which is more informative under class imbalance, capturing the trade-off between precision and recall.

\textbf{Baselines.} We compare DMTFD with several state-of-the-art (SOTA) anomaly detection methods. \textit{DROCC}~\cite{goyal2020drocc} is a robust one-class classification method assuming locally linear manifolds to avoid representation collapse. \textit{DeepSAD}~\cite{ruff2019deep} is a semi-supervised method that detects anomalies via entropy differences in latent distributions. \textit{USAD}~\cite{audibert2020usad} employs adversarially trained autoencoders for unsupervised time-series anomaly detection. \textit{GANF}~\cite{dai2021graph} enhances normalizing flows using Bayesian networks to model inter-series dependencies. \textit{MTGFlow}~\cite{zhou2023detecting} and \textit{MTGFlow Cluster}~\cite{zhou2024label} utilize dynamic graph structure learning and entity-aware flows for fine-grained density estimation, with the latter introducing clustering for improved accuracy. 

\subsection{Fault Detection Performance on Five Benchmark Dataset}

First, we discuss the performance advantages of the DMTFD method by comparing results on five common benchmarks. We list the results for the AUROC and AUPRC metrics in Tables~\ref{tab_my_label_roc} and ~\ref{tab_my_label_pr}, where the numbers in parentheses represent the variance from five different seed experiments. 
{
     In Table \ref{tab_my_label_roc} and Table \ref{tab_my_label_pr}, the results of DROCC, DeepSAD, USAD, DAGMM, GANF, and MTGFlow are from the paper~\cite{zhou2023detecting}\footnote{https://github.com/zqhang/MTGFLOW}. The results of MTGFlow Cluster are from \cite{zhou2024label}\footnote{https://github.com/zqhang/MTGFLOW\_Cluster}.
}
We report on two variants of DMTFD: one (DMTFD) uses the same hyperparameters across all experiments, and the other (DMTFD*) involves hyperparameter tuning specific to each target dataset. The best results are highlighted in bold. To visualize the benefits of DMTFD performance, the ROC and PR curves on four dataset is showned in Fig.~\ref{fig_ROC_PR}. In the ROC and PR curves, the DMTFD method consistently outperforms the baseline methods, demonstrating its robustness and effectiveness in anomaly detection tasks. 

\textbf{Performance improvement.} In terms of performance improvements, as shown in Tables~\ref{tab_my_label_roc} and Tables~\ref{tab_my_label_pr}, the DMTFD method shows significant enhancements compared to state-of-the-art methods. { We compare the performance of different methods using a consistent window size (DMTFD) with our optimal results obtained through varying window sizes (DMTFD*) to showcase the maximum potential. The best results is in \textbf{bold}.}  These improvements are evident across both PR and AUC curves, particularly on the SWAT dataset with a 5.4\% increase in AUROC and a 7.7\% increase in AUPRC, and on the MSL dataset with a 6\% increase in AUROC and a 7.0\% increase in AUPRC. On average, there is a 4.1\% increase in AUROC performance and a 7.0\% increase in AUPRC performance. The consistent improvements in both AUROC and AUPRC indicate that the DMTFD method has enhanced performance across multiple aspects. 

\textbf{Stability improvement.} As shown in Tables~\ref{tab_my_label_roc} and ~\ref{tab_my_label_pr} and Fig.~\ref{fig_ROC_PR}, the stability of the DMTFD method is also noteworthy, exhibiting lower variance across all test datasets compared to traditional methods. This indicates that the DMTFD method is more robust and less sensitive to random initialization, making it a reliable choice for anomaly detection tasks.

\textbf{Performance potential.} Furthermore, by comparing the results of DMTFD and DMTFD*, we find that conducting hyperparameter searches for each dataset can significantly boost model performance. This suggests that our method has good hyperparameter adaptability and considerable potential for further performance enhancements.

\begin{figure*}[t]
    \centering
    \includegraphics[width=1\textwidth]{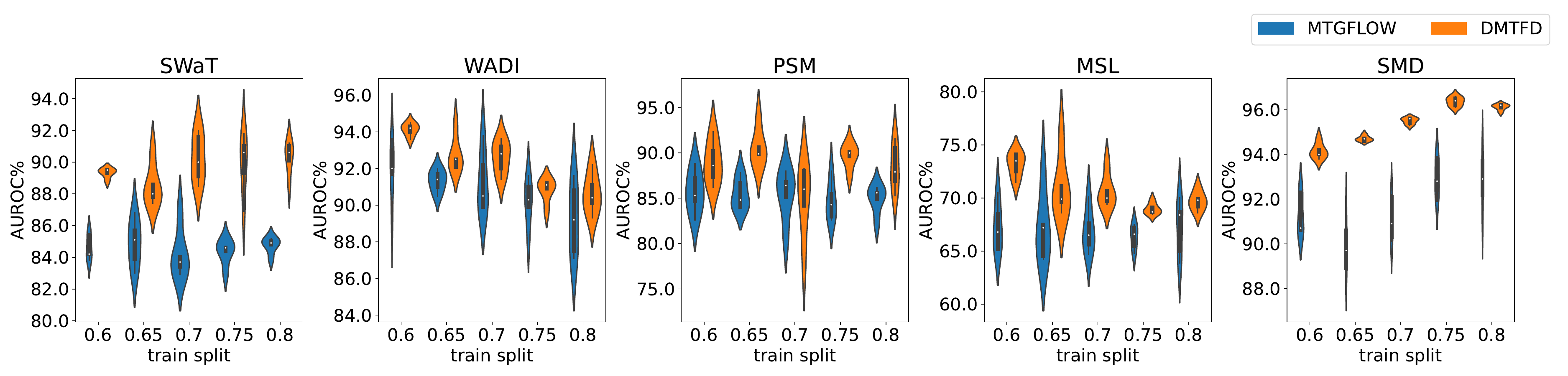}
    \vspace{-0.3cm}
    \caption{\textbf{Violin plot of AUROC scores on different training splits.} The x-axis represents the training split ratio, and the y-axis represents the AUROC scores. The DMTFD method consistently outperforms the baseline methods(MTGFLOW) across different training splits, demonstrating its robustness and effectiveness in anomaly detection tasks.}
    \label{fig_anomaly_ratio}
\end{figure*}
\textbf{Additional results on difference train/validation splits.}
To further investigate the influence of anomaly contamination rates, we vary training splits to adjust anomalous contamination rates. For all the above-mentioned datasets, the training split increases from 60\% to 80\% with 5\% stride. We present an average result over five runs in Fig.~\ref{fig_anomaly_ratio}. Although the anomaly contamination ratio of training dataset rises, the anomaly detection performance of MTGFlow remains at a stable high level. This indicates that the proposed DMTFD method is robust to the anomaly contamination ratio of the training dataset.

\subsection{Visualization of multi-subclass data distributions}

\begin{figure*}[t]
    \centering
    \includegraphics[width=1\textwidth]{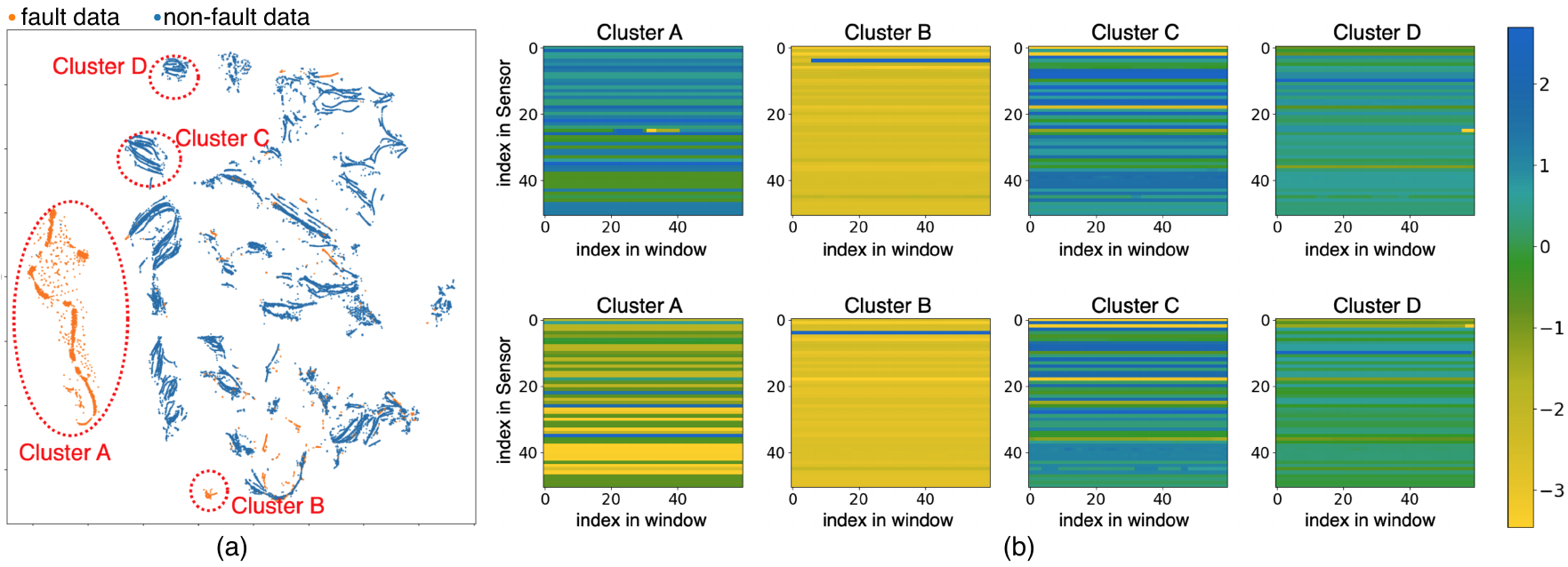}
    \vspace{-0.3cm}
    \caption{\textbf{Visualization of multi-subclass data distributions.} (a) The t-SNE visualization of embedding of DMTFD on SWAT datasets. { The color shown the different states of the data.}
    Multiple clusters can be seen in the representation, which is consistent with the multiple manifold assumption of this paper. The normal and abnormal states are not a single cluster, but can be divided into separate subclusters. (b) The heatmaps of individual data point, each row in the heatmap is the index number of the time and each column is the index number of the sensor. We selected four sub-clusters, and two samples from each cluster were randomly selected for presentation. The similarity of samples within sub-clusters and the difference of samples between sub-clusters are illustrated.}
    \label{fig_vis_cluster}
\end{figure*}

\begin{figure}
    \centering
    \includegraphics[width=1\columnwidth]{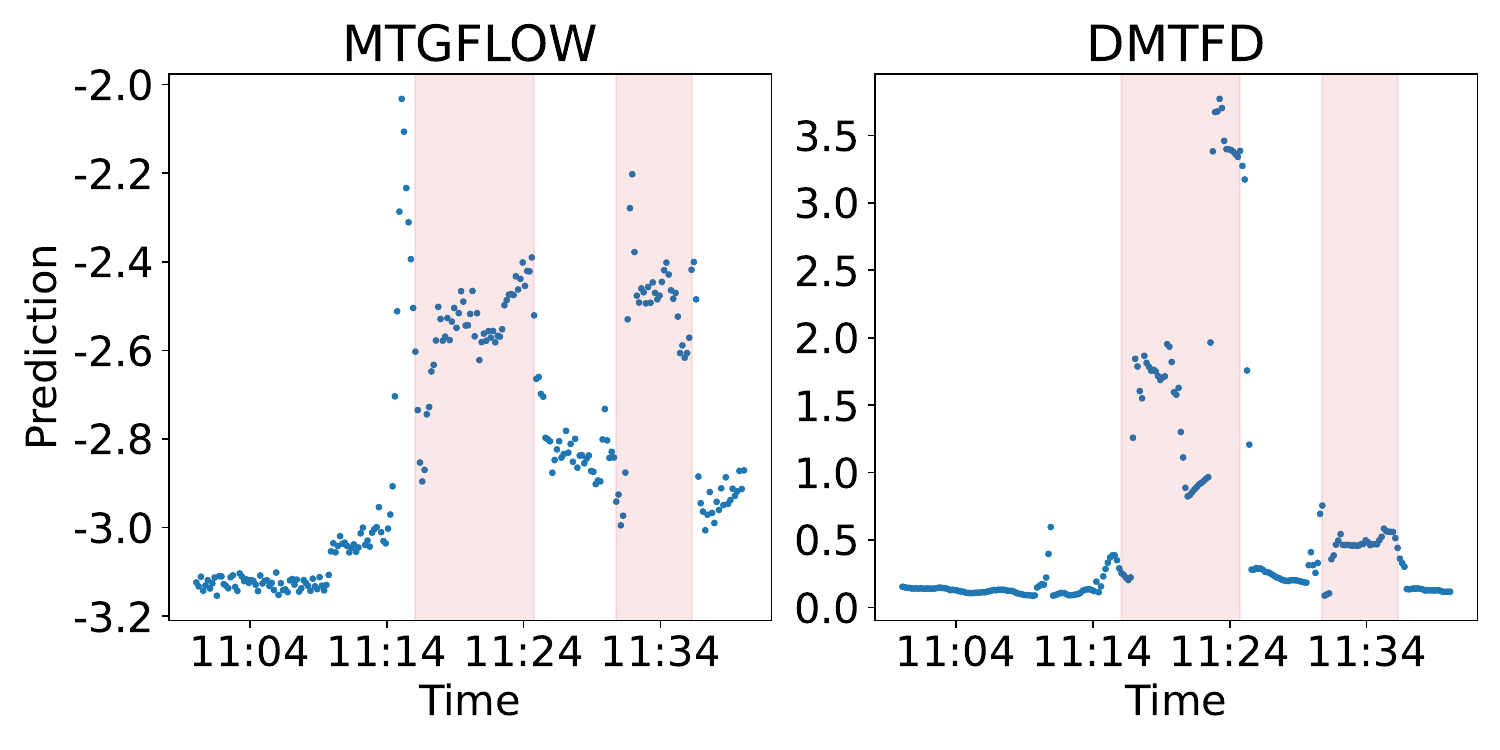}
    \vspace{-0.3cm}
    \caption{\textbf{Point plot of the anomalies predict outputs of MTGFLOW and proposed DMTFD.} The x-axis represents the anomaly index, and the y-axis represents the log-likelihood of the anomaly. Anomalous ground truths are marked by a red background. This indicates that DMTFD is able to detect anomalies more accurately and in a timely manner.}
    \label{fig_LogLikelihoods}
\end{figure}

To further investigate the effectiveness of the DMTFD method, we visualize the multi-subclass data distributions on the SWAT dataset. As shown in Fig.~\ref{fig_vis_cluster}, the t-SNE visualization of the embedding of DMTFD on SWAT datasets reveals multiple clusters, consistent with the multiple manifold assumption of this paper. The normal and abnormal states are not a single cluster but can be divided into separate subclusters. The heatmaps of individual data points further illustrate the similarity of samples within sub-clusters and the differences between samples in different sub-clusters. This visualization demonstrates the ability of the DMTFD method to capture the diverse patterns present in both normal and abnormal states, enhancing its anomaly detection capabilities.

In addition, we present the point plot of the anomalies predict outputs of MTGFLOW and the proposed DMTFD in Fig.~\ref{fig_LogLikelihoods}. The x-axis represents the anomaly index, and the y-axis represents the log-likelihood of the anomaly. Anomalous ground truths are marked by a red background. This indicates that DMTFD is able to detect anomalies more accurately and in a timely manner, outperforming MTGFLOW in terms of anomaly detection performance.

Fig.~\ref{fig_comparison_anomaly_scores} shows the bar plot of the frequency of normalized anomaly scores on baseline methods and DMTFD. The x-axis represents the anomaly scores, and the y-axis represents the frequency of each score. The normalized anomaly scores of DMTFD are significantly lower than those of GANF and MTGFlow, indicating that DMTFD is more effective at distinguishing between normal and abnormal states.

\begin{figure*}
    \centering
    \includegraphics[width=1.9\columnwidth]{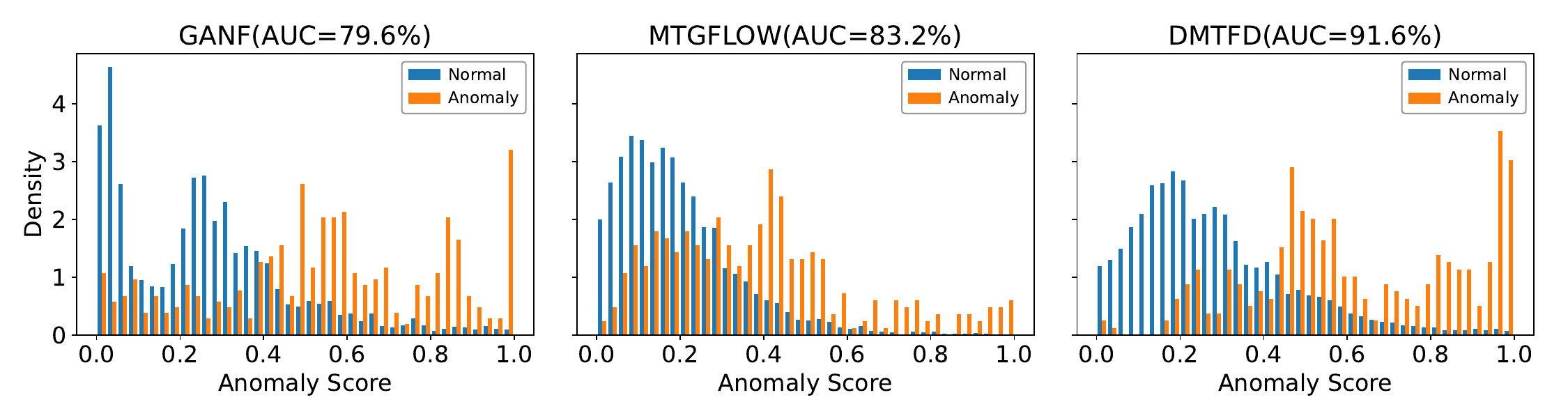}
    \vspace{-0.3cm}
    \caption{\textbf{The bar plot of the frequency of normalized anomaly scores on baseline methods and DMTFD.} The x-axis represents the anomaly scores, and the y-axis represents the frequency of each score. The normalized anomaly scores of DMTFD are significantly lower than those of GANF and MTGFlow, indicating that DMTFD is more effective at distinguishing between normal and abnormal states.}
    \label{fig_comparison_anomaly_scores}
\end{figure*}

\subsection{Ablation Study: The Effect of SoftCL Loss.}

\begin{table}[t]
    \caption{\textbf{Ablation study.} AUROC comparison on four datasets~(SWaT, WADI, PSM, and MSL). The best results are in bold. The results are averaged over 5 runs, and the standard deviation is shown in parentheses.}
    \centering
    \small
    \begin{tabular}{c|cccccc|c|cccccc|c}
        \toprule
        {Method}                                                                                            & \multirow{1}{*}{SWaT}       & \multirow{1}{*}{WADI }      & \multirow{1}{*}{PSM}        & \multirow{1}{*}{MSL}        \\
        \midrule

        \textit{DMTFD}                                                                                       & \textbf{90.5(\textpm{0.9})} & \textbf{94.3(\textpm{0.4})} & \textbf{89.2(\textpm{2.4})} & \textbf{75.0(\textpm{2.0})} \\
        \textit{w. Gaussian}                                                                                & 88.9(\textpm{1.1})          & 93.9(\textpm{1.7})          & 86.7(\textpm{2.5})          & 72.1(\textpm{0.7})          \\
        \textit{w/o. MML}                                                                                & 84.3(\textpm{1.5})          & 91.7(\textpm{1.2})          & 85.5(\textpm{1.3})          & 70.7(\textpm{2.4})          \\
        \textit{w/o. CL}                                                                                    & 84.7(\textpm{1.3})          & 90.8(\textpm{1.4})          & 85.9(\textpm{1.6})          & 69.2(\textpm{2.1})          \\
        \multicolumn{1}{c|}{\begin{tabular}[c]{@{}c@{}}\textit{w/o. Aug}\\ \textit{(MTGFlow)}\end{tabular}} & 84.8(\textpm{1.5})          & 91.9(\textpm{1.1})          & 85.7(\textpm{1.5})          & 67.2(\textpm{1.7})          \\
        \toprule
    \end{tabular}
    \label{tb_ablation}
\end{table}

To assess the effectiveness of each component designed in our model, we conducted a series of ablation experiments. The results of these experiments are presented in Table~\ref{tb_ablation}.

We performed controlled experiments to verify the necessity of the Soft Contrastive Learning (SoftCL) Loss in the UNDA settings across all four datasets. The performance of the proposed DMTFD method is denoted as {\textit{DMTFD}} in the table. The variant {\textit{w/o. SoftCL}} represents the model performance with the SoftCL loss component, $\mathcal{L}_{MML}(\mathbf{x}_{i}^{t}, {y_{i}}^{t})$, removed from the overall loss function of DMTFD. The variant {\textit{w. CL}} indicates the model performance when the SoftCL loss is replaced by a typical Contrastive Loss (CL), $\mathcal{L}_\text{CCL}$, as defined in Eq.~(\ref{eq_CL}). 
The results clearly demonstrate that the SoftCL Loss significantly outperforms the traditional CL loss. We attribute the inferior performance of $\mathcal{L}_\text{CCL}$ to its inability to adequately address the view-noise caused by domain bias.

\subsection{Hyperparameter Robustness}

\textbf{Hyperparameter Robustness: Window Size and Number of Blocks.} 
The Table.~\ref{tab_ablation_hyperparameters_window_blocks} presents an ablation study investigating the robustness of hyperparameters, focusing on the window size and number of blocks. Three distinct datasets, SWaT, WADI, and PSM, along with MSL, are evaluated using DMTFD. Our method showcases promising results, with consistently competitive AUROC scores across varying hyperparameter configurations. Notably, the standard deviations accompanying AUROC values indicate a high degree of stability, underscoring the reliability of our approach. Analysis of the table reveals that optimal configurations often coincide with larger window sizes and moderate block numbers, suggesting a preference for capturing broader temporal contexts while maintaining computational efficiency. Additionally, trends indicate that as the number of blocks increases, there's a discernible improvement in performance, albeit with diminishing returns beyond a certain threshold. This observation highlights the importance of carefully balancing model complexity with computational resources. Overall, the DMTFD method demonstrates robustness and effectiveness in anomaly detection tasks.

\textbf{Hyperparameter Robustness: Learning rate, $\nu$, and number of epoch.}
In order to further investigate the effectiveness of MTGFlow, we give a detailed analysis based on SWaT dataset. We conduct ablation studies on the learning rate, $\nu$, and the number of epochs. {The results are shown in Table.~\ref{tab_ablation_lr}, Table.~\ref{tab_abalation_v_laten}, Table.~\ref{tab_abalation_epoch} and Table.~\ref{tab_ablation_topological_loss}. The results show that the performance of MTGFlow is relatively stable across different learning rates, $\nu$, and the number of epochs. This indicates that MTGFlow is robust to hyperparameter changes and can achieve good performance with a wide range of hyperparameters.}

\begin{table*}
   \centering
   \small
   \caption{\textbf{Parameters analysis: learning rates.} Table of AUC scores (AUROC\%) for different learning rates on the SWaT dataset, indicating DMTFD's sensitivity to the learning rate.}
    \vspace{-0.2cm}
   \begin{tabular}{c|cccccc}
      \toprule
      lr    & 0.001              & 0.005              & 0.01               & 0.05               & 0.1                & 0.5                \\
      \midrule
      AUROC & 84.5(\textpm{4.2}) & \textbf{90.9(\textpm{0.9})} & 90.2(\textpm{0.9}) & 82.3(\textpm{1.5}) & 82.4(\textpm{2.7}) & 79.2(\textpm{4.1}) \\
      \bottomrule
   \end{tabular}
   \label{tab_ablation_lr}
\end{table*}

\begin{table*}
   \centering
   \small
   \caption{\textbf{Parameters analysis: number of epochs.}  Table of AUC scores (AUROC\%) for different epochs on the SWaT dataset, indicating DMTFD's sensitivity to the number of epochs.}
    \vspace{-0.2cm}
   \begin{tabular}{c|cccccc}
      \toprule
      epoch & 40                 & 100                & 200                & 400                & 500                & 1000               \\
      \midrule
      AUROC & 81.4(\textpm{2.9}) & 85.8(\textpm{2.9}) & 89.8(\textpm{1.4}) & \textbf{90.2(\textpm{0.9})} & \textbf{90.2(\textpm{0.9})} & \textbf{90.2(\textpm{0.9})} \\
      \bottomrule
   \end{tabular}
   \label{tab_abalation_epoch}
\end{table*}

\begin{table*}
   \centering
   \small
   \caption{\textbf{Parameters analysis: hyperparameter $\nu$.}  Table of (AUROC\%) scores (AUROC\%) for different $\nu$ on the SWaT dataset, indicating DMTFD's sensitivity to the hyperparameter $\nu$.}
    \vspace{-0.2cm}
   \begin{tabular}{c|cccc}
      \toprule
      $\nu$ & 0.005              & 0.01               & 0.05               & 0.1                \\
      \midrule
      AUROC & 88.4(\textpm{1.8}) & 90.2(\textpm{0.9}) & 90.5(\textpm{1.3}) & \textbf{90.6(\textpm{2.6})} \\
      \bottomrule
   \end{tabular}
   \label{tab_abalation_v_laten}
\end{table*}

\begin{table*}
   \centering
   \small
   \caption{\textbf{Parameters analysis: wight of negative~(Ne) sample and positive~(Po) sample in Eq.(\ref{equ_Lscl})}. Table of AUC scores (AUROC\%) for different topological loss ratios on the SwaT dataset.}
    \vspace{-0.2cm}
   \begin{tabular}{c|cccccc}
      \toprule
      \multicolumn{1}{c|}{\diagbox{ne}{po}} & 0.5                & 1                  & 2                  & 3                  & 4                  & 5                  \\\midrule
      1                                     & 91.2(\textpm{1.0}) & 90.2(\textpm{3.1}) & 91.3(\textpm{1.7}) & 89.8(\textpm{1.8}) & 90.9(\textpm{1.1}) & 89.8(\textpm{1.8}) \\\midrule
      2                                     & 89.0(\textpm{2.0}) & 90.0(\textpm{2.3}) & 90.1(\textpm{0.7}) & 90.6(\textpm{1.6}) & 90.6(\textpm{2.3}) & 91.2(\textpm{1.9}) \\\midrule
      3                                     & 89.8(\textpm{2.8}) & 90.3(\textpm{2.5}) & 91.2(\textpm{1.7}) & 88.6(\textpm{2.8}) & 91.3(\textpm{0.9}) & 90.1(\textpm{1.8}) \\\midrule
      4                                     & 89.9(\textpm{1.1}) & 89.1(\textpm{2.6}) & 90.3(\textpm{1.0}) & 88.7(\textpm{2.5}) & 90.1(\textpm{0.8}) & 90.4(\textpm{2.7}) \\\midrule
      5                                     & 89.2(\textpm{3.2}) & 90.2(\textpm{0.9}) & 92.1(\textpm{1.1}) & 89.0(\textpm{2.2}) & 90.2(\textpm{0.8}) & 90.3(\textpm{1.1}) \\\midrule
      6                                     & 89.3(\textpm{0.9}) & 90.9(\textpm{1.4}) & 89.4(\textpm{1.2}) & 90.0(\textpm{1.3}) & 89.5(\textpm{2.2}) & 90.4(\textpm{1.9}) \\
      \bottomrule
   \end{tabular}
   \label{tab_ablation_topological_loss}
\end{table*}

\section{Conclusion}
\label{sec:conclusion}

    In this work, we propose DMTFD, an unsupervised framework for fault detection in multivariate time series. To overcome the limitations of Gaussian assumptions, DMTFD models data with multi-Gaussian representations. A neighbor-based augmentation strategy is designed to generate positive pairs for learning, while a representation head enables  optimization based on local streamform similarities. This enhances the separation between normal and anomalous states and broadens anomaly coverage. Experiments on standard benchmarks (e.g., SWaT, WADI) demonstrate that DMTFD achieves superior AUC and PR scores, with lower false alarm rates and improved detection accuracy. Visualizations also validate the effectiveness of the multi-Gaussian modeling. {
    \color{black}Although our current focus is unsupervised fault detection, DMTFD can be extended to Remaining Useful Life (RUL) prediction. Its pretrained embeddings provide strong priors for temporal degradation modeling, and with minimal finetuning, can support few-shot RUL forecasting. We leave multi-task extensions integrating fault detection and RUL estimation for future work.}

\section*{Acknowledgments}
This work was supported by `Pioneer' and `Leading Goose' R\&D Program of Zhejiang (2024C01140), Key Research and Development Program of Hangzhou (2023SZD0073), Beijing Natural Science Foundation (L221013), and InnoHK program.

\bibliographystyle{ieee_fullname}
\bibliography{lib_zzl}

\section*{Brief biography of each author}

\begin{IEEEbiography}
[{\includegraphics[width=1in,height=1.0in,clip,keepaspectratio]{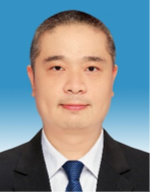}}]{Hong Liu}
    Hong Liu received an Ph.D. from Zhejiang University. He had a postdoctoral experience in electronic science major, Zhejiang University. He was a visiting researcher with The Hong Kong University of Science and Technology. He is currently an associate professor in the School of Information and Electric Engineering, Hangzhou City University. His major research interests include system modeling, optimization and control.
\end{IEEEbiography}
\vspace{-12mm}

\begin{IEEEbiography}
[{\includegraphics[width=1in,height=1.0in,clip,keepaspectratio]{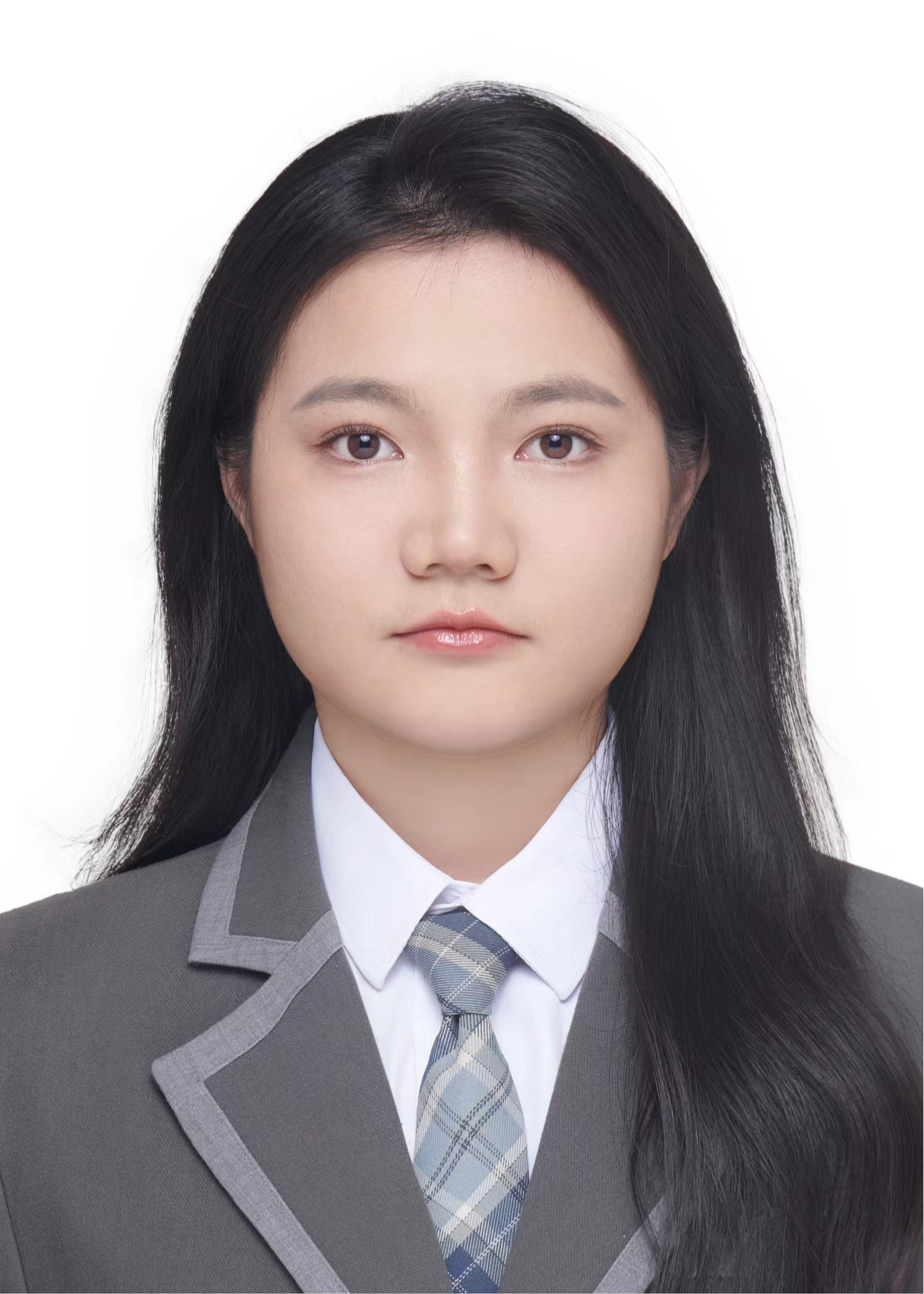}}]{Xiuxiu Qiu}
   Qiu Xiuxiu holds a master's degree in Electronic Information from Zhejiang University of Technology. Her research focuses on affective computing and intelligent information processing.
\end{IEEEbiography}
\vspace{-12mm}

\begin{IEEEbiography}
[{\includegraphics[width=1in,height=1.0in,clip,keepaspectratio]{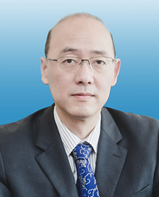}}]{Yiming Shi}
    Yiming Shi holds a Master's degree from Zhejiang University. With long-term dedication to industrial control systems research, he specializes in fieldbus technologies, information security, and edge intelligence.
\end{IEEEbiography}
\vspace{-12mm}

\begin{IEEEbiography}[{\includegraphics[width=1in,height=1.0in,clip,keepaspectratio]{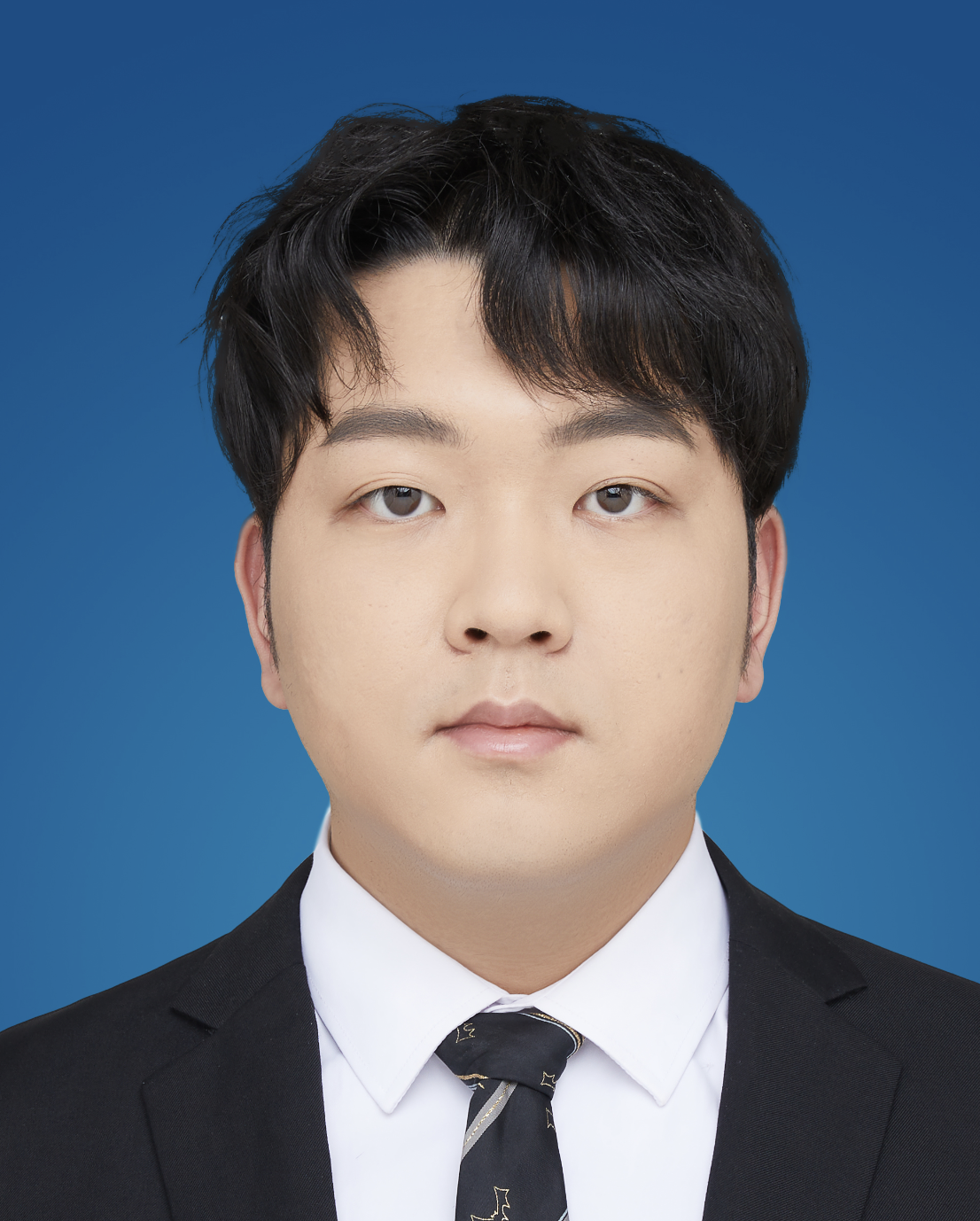}}]{Miao Xu}
    Miao Xu received the M.S. degree from the Institute of Automation, Chinese Academy of Sciences, under the supervision of Zhen Lei. His research interests include computer vision and 3D reconstruction. 
\end{IEEEbiography}

\begin{IEEEbiography}[{\includegraphics[width=1in,height=1.0in,clip,keepaspectratio]{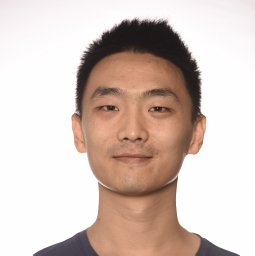}}]{Zelin Zang}
   Zelin Zang received his Ph.D. degree from Zhejiang University, China, under the supervision of Prof. Stan Z. Li. 
    His research focuses on manifold learning, dimensionality reduction, and geometric deep learning, with applications in single-cell omics, protein function analysis, and biomedical image understanding. He has made notable contributions to deep manifold transformation techniques and has proposed several high-impact models in AI for Science, including MuST for spatial transcriptomics integration and DMT-HI for interpretable manifold visualization. His current research interests lie in developing interpretable and structure-aware representation learning methods for high-dimensional biological data, with an emphasis on tree-like structure inference, large-scale biomedical foundation models, and multi-agent medical reasoning systems.
\end{IEEEbiography}
\vspace{-12mm}

\begin{IEEEbiography}[{\includegraphics[width=1in,height=1.0in,clip,keepaspectratio]{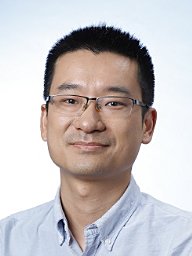}}]{Zhen Lei}
    Zhen Lei (Fellow, IEEE) received the B.S. degree in automation from the University of Science and Technology of China, in 2005, and the Ph.D. degree from the Institute of Automation, Chinese Academy of Sciences, in 2010. He is currently a Professor with the Institute of Automation, Chinese Academy of Sciences. He has published over 200 papers in international journals and conferences with more than 30000 citations in Google Scholar and an H-index of 81. His research interests are in computer vision, pattern recognition, image processing, and face recognition in particular. He is an IAPR Fellow and an AAIA Fellow. He was a winner of the 2019 IAPR Young Biometrics Investigator Award. He was the Program Co-Chair of IJCB2023, the Competition Co-Chair of IJCB2022, and the area chair of several conferences. He is an Associate Editor of IEEE Transactions on Information Forensics and Security; IEEE Transactions on Biometrics, Behavior, and Identity Science; Pattern Recognition; Neurocomputing; and IET Computer Vision.
\end{IEEEbiography}
\vspace{-12mm}

\clearpage
\onecolumn
\appendix

{\section*{A. Details of Dataset}} \label{app_dataset}
We selected the five most commonly used datasets for fault detection to evaluate the effectiveness of our method~{(in Table \ref{tab_dataset_setting})}. These datasets are widely recognized in the field of One-Class Classification (OCC) for Multivariate Time Series (MTS) anomaly detection. The datasets include:

\begin{itemize}
    \item SWaT (Secure Water Treatment)~\cite{goh2016dataset}: Originating from a scaled-down version of an industrial water treatment plant, the SWaT dataset is used for research in cybersecurity and anomaly detection in critical infrastructure systems. SWaT collects 51 sensor data from a real-world industrial water treatment plant, at the frequency of one second. The dataset provides ground truths of 41 attacks launched during 4 days. \footnote[1]{https://itrust.sutd.edu.sg/itrust-labs\_datasets/dataset\_info}.
    \item WADI (Water Distribution)~\cite{ahmed2017wadi}: The WADI dataset is collected from a water distribution testbed, simulating real-world water supply networks. It is particularly useful for studying the impact of cyber-attacks on public utility systems. WADI collects 121 sensor and actuators data from WADI testbed, at the frequency of one second. The dataset provides ground truths of 15 attacks launched during 2 days. \footnote[2]{https://itrust.sutd.edu.sg/itrust-labs\_datasets/dataset\_info/}.
    \item PSM (Pooled Server Metrics)~\cite{abdulaal2021practical}: This dataset aggregates performance metrics from multiple server nodes managed by eBay. It is utilized to identify outliers indicating potential issues in server performance or security breaches \footnote[3]{https://github.com/tuananhphamds/MST-VAE}.
    \item MSL (Mars Science Laboratory rover)~\cite{hundman2018detecting}: This dataset originates from the Mars Science Laboratory rover, specifically the Curiosity rover. It includes telemetry data used to detect anomalies in the rover's operational parameters during its mission on Mars \footnote[4]{https://github.com/khundman/telemanom}.
    \item SMD (Server Machine Dataset)~\cite{su2019robust}: Collected from a large internet company's server farms, this dataset comprises metrics from different server machines, providing a basis for detecting unusual behaviors in server operations \footnote[5]{https://github.com/NetManAIOps/OmniAnomaly}. { For this dataset, the average results across all entities are reported. For consistency and proper evaluation, we adjusted the random seed during dataset splitting to ensure anomalies are present in the test set while maintaining a training/testing ratio. }
\end{itemize}

\begin{table*}[h]
    \footnotesize
    \centering
    \caption{\textbf{Dataset information.} The number of metrics, training size, test size, training anomaly ratio, and test anomaly ratio of each dataset.}
    \begin{tabular}{l|cccccc}
        \toprule
        Property                    & SWaT    & WADI    & PSM    & MSL    & SMD      \\
        \midrule
        \# of Metrics               & 51      & 123     & 25     & 55     & 38       \\
        Training Size               & 269,951 & 103,680 & 52,704 & 44,237 & 425,052  \\
        Test Size                   & 89,984  & 69,121  & 35,137 & 29,492 & 283,368  \\
        Training \& Test Size       & 359,935 & 172,801 & 87,841 & 73,729 & 708,420  \\
        Training Anomaly Ratio (\%) & 17.7    & 6.4     & 23.1   & 14.7   & 4.2      \\
        Test Anomaly Ratio (\%)     & 5.2     & 4.6     & 34.6   & 4.3    & 4.1      \\
        \bottomrule
    \end{tabular}
    \label{tab_dataset_setting}
\end{table*}

{\section*{A. Details of Baselines}} \label{app_baseline}
\textbf{Baselines.} {We compare our method, DMTFD, against state-of-the-art (SOTA) methods.

These methods include:
\textit{DROCC}~\cite{goyal2020drocc}: A robust one-class classification method applicable to most standard domains without requiring additional information. It assumes that points from the class of interest lie on a well-sampled, locally linear low-dimensional manifold, effectively avoiding the issue of representation collapse

\textit{DeepSAD}~\cite{ruff2019deep}: A general semi-supervised deep anomaly detection method that leverages a small number of labeled normal and anomalous samples during training. This method is based on an information-theoretic framework that detects anomalies by comparing the entropy of the latent distributions between normal and anomalous data.

\textit{USAD}~\cite{audibert2020usad}: An unsupervised anomaly detection method for multivariate time series based on adversarially trained autoencoders. This method is capable of fast and stable anomaly detection, leveraging its autoencoder architecture for unsupervised learning and adversarial training to isolate anomalies efficiently.

\textit{GANF}~\cite{dai2021graph}: An unsupervised anomaly detection method for multiple time series. GANF improves normalizing flow models by incorporating a Bayesian network among the constituent series, enabling high-quality density estimation and effective anomaly detection.

\textit{MTGFlow}~\cite{zhou2023detecting}: An unsupervised anomaly detection approach for multivariate time series. MTGFlow leverages dynamic graph structure learning and entity-aware normalizing flows to capture the mutual and dynamic relations among entities and provide fine-grained density estimation

\textit{MTGFlow Cluster}~\cite{zhou2024label}: An unsupervised anomaly detection approach for multivariate time series. MTGFlow leverages dynamic graph structure learning and entity-aware normalizing flows to estimate the density of training samples and identify anomalous instances. Additionally, a clustering strategy is employed to enhance density estimation accuracy.

}
We evaluate the performance degradation of semi-supervised methods when dealing with contaminated training sets and compare our DMTFD with SOTA unsupervised density estimation methods.

{\section*{A. Details of SoftCL loss}} \label{app_proof_1}

\subsection*{A.1 Details of the transformation from Eq.~(\ref{eq_CL}) to Eq.~(\ref{eq_nce2})}

We start with $L_\text{CL} =  - \log \frac{ \exp(S(z_i, z_j))}{\sum_{k=1}^{N_K} \exp(S(z_i, z_k))}$ (Eq.~(\ref{eq_CL})), then

$$ L_\text{CL} = \log N_K - \log \frac{ \exp(S(z_i, z_j))}{\frac{1}{N_K}\sum_{k=1}^{N_K} \exp(S(z_i, z_k))}. $$

We are only concerned with the second term that has the gradient. Let $(i,j)$ are positive pair and $(i,k_1), \cdots, (i,k_N) $ are negative pairs. The overall loss associated with point $i$ is:
\begin{equation*}
  \begin{aligned}
      & - \log \frac
    {\exp(S(z_i, z_j))}
    { \frac{1}{N_K} \sum_{k=1}^{N_K} \exp(S(z_i, z_k))}                                                          \\
    = & - \left[
      \log \exp(S(z_i, z_j)) - \log
    { \frac{1}{N_K} \sum_{k=1}^{N_K} \exp(S(z_i, z_k))} \right]                                                  \\
    = & - \left[
      \log \exp(S(z_i, z_j)) - \sum_{k=1}^{N_K} \log \exp(S(z_i, z_{k})) +
    \sum_{k=1}^{N_K} \log \exp(S(z_i, z_{k})) - \log { \frac{1}{N_K} \sum_{k=1}^{N_K} \exp(S(z_i, z_k))} \right] \\
    = & - \left[
      \log \exp(S(z_i, z_j)) - \sum_{k=1}^{N_K} \log \exp(S(z_i, z_{k})) +
    \log \Pi_{k=1}^{N_K}  \exp(S(z_i, z_{k})) - \log { \frac{1}{N_K} \sum_{k=1}^{N_K} \exp(S(z_i, z_k))} \right] \\
    = & - \left[
      \log \exp(S(z_i, z_j)) - \sum_{k=1}^{N_K} \log \exp(S(z_i, z_{k})) +
    \log \frac {\Pi_{k=1}^{N_K}  \exp(S(z_i, z_{k}))}{ \frac{1}{N_K} \sum_{k=1}^{N_K} \exp(S(z_i, z_k))} \right] \\
  \end{aligned}
\end{equation*}

We focus on the case where the similarity is normalized, $S(z_i, z_k) \in [0,1]$. The data $i$ and data $k$ is the negative samples, then $S(z_i, z_k)$ is near to $0$, $\exp(S(z_i, z_{k}))$ is near to $1$, thus the $\frac {\Pi_{k=1}^{N_K} \exp(S(z_i, z_{k}))}{ \frac{1}{N} \sum_{k=1}^{N_K} \exp(S(z_i, z_k))}$ is near to 1, and $\log \frac {\Pi_{k=1}^{N_K}  \exp(S(z_i, z_{k}))}{ \frac{1}{N} \sum_{k=1}^{N_K} \exp(S(z_i, z_k))}$ near to 0. We have

\begin{equation*}
  \begin{aligned}
    L_\text{CL}
     & \approx  - \left[
    \log \exp(S(z_i, z_j)) - \sum_{k=1}^{N_K} \log \exp(S(z_i, z_{k})) \right] \\
  \end{aligned}
\end{equation*}

We denote $ij$ and $ik$ by a uniform index and use $\mathcal{H}_{ij}$ to denote the homology relation of $ij$.

\begin{equation*}
  \begin{aligned}
    L_\text{CL}
     & \approx - \left[
    \log \exp(S(z_i, z_j)) - \sum_{k=1}^{N_K} \log \exp(S(z_i, z_{k})) \right]                                       \\
     & \approx - \left[
    \mathcal{H}_{ij} \log \exp(S(z_i, z_j)) - \sum_{j=1}^{N_K} (1-\mathcal{H}_{ij}) \log \exp(S(z_i, z_{j})) \right] \\
     & \approx - \left[
      \sum_{j=1}^{N_K+1} \left\{  \mathcal{H}_{ij} \log \exp(S(z_i, z_j)) +  (1-\mathcal{H}_{ij}) \log \{\exp(-S(z_i, z_{j}))\}
      \right\}
    \right]                                                                                                          \\
  \end{aligned}
\end{equation*}

we define the similarity of data $i$ and data $j$ as $Q_{ij} = \exp(S(z_i, z_j))$ and the dissimilarity of data $i$ and data $j$ as $\dot{Q}_{ij} =  \exp(-S(z_i, z_j))$.

\begin{equation*}
  \begin{aligned}
    L_\text{CL} \approx - \left[
      \sum_{j=1}^{N_K+1} \left\{  \mathcal{H}_{ij} \log Q_{ij} +  (1-\mathcal{H}_{ij}) \log \dot{Q}_{ij}
      \right\}
    \right] \\
  \end{aligned}
\end{equation*}

\subsection*{A.2 The proposed SoftCL loss is a smoother CL loss}

This proof tries to indicate that the proposed SoftCL loss is a smoother CL loss. We discuss the differences by comparing the two losses to prove this point. 
the forward propagation of the network is,
${z}_{i}=H(\hat{z}_{i}), \hat{z}_{i} =F(x_{i})$, 
${z}_{j}=H(\hat{z}_{j}), \hat{z}_{j} =F(x_{j})$.
We found that we mix $y$ and $\hat{z}$ in the main text, and we will correct this in the new version. So, in this section 
${z}_{i}=H(y_{i}), y_{i} =F(x_{i})$, 
${z}_{j}=H(y_{j}), y_{j} =F(x_{j})$ is also correct.

Let $H(\cdot)$ satisfy $K$-Lipschitz continuity, then
$
  d^z_{ij} = k^* d^y_{ij} , k^* \in [1/K, K],
$
where $k^*$ is a Lipschitz constant. The difference between $L_\text{SoftCL}$ loss and $L_\text{CL}$ loss is,
\begin{equation}
  \begin{aligned}
    L_{\text{CL}}- L_\text{SoftCL} \approx
    \sum_j \biggl[
    \left(
    \mathcal{H}_{ij} - [1+(e^\alpha -1)\mathcal{H}_{ij}] \kappa \left(d^y_{ij} \right)
    \right)
    \log
    \left(
    \frac
    {1}
    {\kappa \left( d_{ij}^{z}\right)}
    -
    1
    \right)
    \biggl] .
  \end{aligned} \label{eq_SoftCLcl}
\end{equation}
Because the  $\alpha > 0$, the proposed SoftCL loss is the soft version of the CL loss. if $\mathcal{H}_{ij}=1$, we have:

\begin{equation}
  \begin{aligned}
    (L_{\text{CL}} - L_{\text{SoftCL}})  |_{\mathcal{H}_{ij} =1} = \sum
    \biggl[
      \left(
      (1 - e^\alpha) \kappa \left( k^* d^z_{ij} \right)
      \right)
      \log
      \left(
      \frac
      {1}
      {\kappa \left( d_{ij}^{z}\right)}
      -
      1
      \right)
    \biggl] \\
  \end{aligned}
\end{equation}

then:

\begin{equation}
  \begin{aligned}
       \lim_{\alpha \to 0}
    ( L_{\text{CL}} - L_{\text{SoftCL}} ) |_{\mathcal{H}_{ij} =1}
    = \lim_{\alpha \to 0} \sum
    \biggl[
      \left(
      (1 - e^\alpha) \kappa \left( k^* d^z_{ij} \right)
      \right)
      \log
      \left(
      \frac
      {1}
      {\kappa \left( d_{ij}^{z}\right)}
      -
      1
      \right)
    \biggl]              = 0
  \end{aligned}
  \label{eq:lim}
\end{equation}

Based on Eq.(\ref{eq:lim}), we find that if $i,j$ is neighbor~($\mathcal{H}_{ij}=1$) and $\alpha\to0$, there is no difference between the CL loss $L_\text{CL}$ and SoftCL loss $L_{\text{SoftCL}}$.
When if $\mathcal{H}_{ij}=0$, the difference between the loss functions will be the function of $d_{ij}^{z}$. The CL loss $L_\text{CL}$ only minimizes the distance between adjacent nodes and does not maintain any structural information. The proposed SoftCL loss considers the knowledge both comes from the output of the current bottleneck and data augmentation, thus less affected by view noise.

\vspace{5mm}

\textbf{Details of Eq.~(\ref{eq_SoftCLcl}).}
Due to the very similar gradient direction, we assume $\dot{Q}_{ij} = 1-Q_{ij}$. The contrastive learning loss is written as, 
\begin{equation}
  \begin{aligned}
    L_\text{CL}  \approx & - \sum
    \left\{
    \mathcal{H}_{ij}
    \log
    Q_{ij}
    +
    \left(1-\mathcal{H}_{ij} \right)
    \log
    \left(1-{Q}_{ij} \right)
    \right\}
  \end{aligned}
\end{equation}
where $\mathcal{H}_{ij}$ indicates whether $i$ and $j$ are augmented from the same original data. 

The SoftCL loss is written as:

\begin{equation}
  \begin{aligned}
    L_{\text{SoftCL}} & =
    -
    \sum
    \left\{
    P_{ij}
    \log
    Q_{ij}
    +
    \left(1-P_{ij}\right)
    \log
    \left(1-Q_{ij}\right)
    \right\}
  \end{aligned}
  \label{eq:appendix_SoftCL}
\end{equation}

According to Eq.~(4) and Eq.~(5), we have

\begin{equation}
  \begin{aligned}
    P_{ij} &= R_{ij} \kappa(d^y_{ij}) = R_{ij} \kappa(y_i, y_j),
       R_{ij} = \left\{
    \begin{array}{lr}
       e^\alpha   \;\;\; \text{if} \;\; \mathcal{H}(x_i, x_j)=1 \\
      1  \;\;\;\;\;\;\;  \text{otherwise}                       \\
    \end{array}
    \right.,                    \\
    Q_{ij} & = \kappa(d_{ij}^z) = \kappa(z_i, z_j),
  \end{aligned}
\end{equation}

For ease of writing, we use distance as the independent variable, $d_{ij}^y=\|y_i- y_j\|_2$, $d_{ij}^z=\|z_i- z_j\|_2$.

The difference between the two loss functions is:

\begin{equation}
  \begin{aligned}
       & L_\text{CL} - L_{\text{SoftCL}} \\
       =& -\sum\biggl[
    \mathcal{H}_{ij}
    \log \kappa \left( d_{ij}^{z}\right)
    +
    \left(1-\mathcal{H}_{ij} \right)
    \log \left(1-\kappa \left(d_{ij}^{z}\right)\right)
    -
     R_{ij}\kappa\left( d^y_{ij} \right)
    \log \kappa \left( d^z_{ij} \right)
    -
    \left(1- R_{ij}\kappa\left( d^y_{ij} \right)\right)
    \log \left(1-\kappa \left( d^z_{ij} \right)\right)
    \biggl]                   \\
    = & -\sum\biggl[
      \left(
      \mathcal{H}_{ij} -  R_{ij}\kappa\left(d^y_{ij} \right)
      \right)
      \log \kappa \left( d_{ij}^{z}\right)
      +
      \left(
      1-\mathcal{H}_{ij} -1 +  R_{ij}\kappa\left(d^y_{ij} \right)
      \right)
      \log \left(1-\kappa \left(d^z_{ij}\right)\right)
    \biggl]                 \\
    = & -\sum\biggl[
      \left(
      \mathcal{H}_{ij} - R_{ij} \kappa \left(d^y_{ij} \right)
      \right)
      \log \kappa \left( d_{ij}^{z}\right)
      +
      \left(
      R_{ij} \kappa \left(  d^y_{ij} \right) - \mathcal{H}_{ij}
      \right)
      \log \left(1-\kappa \left(d^z_{ij}\right)\right)
    \biggl]                 \\
    = & -\sum\biggl[
      \left(
      \mathcal{H}_{ij} - R_{ij}\kappa \left(d^y_{ij} \right)
      \right)
      \left(
      \log \kappa \left( d_{ij}^{z}\right)
      -
      \log \left(1-\kappa \left(d^z_{ij}\right)\right)
      \right)
    \biggl]                 \\
    = & \sum\biggl[
      \left(
      \mathcal{H}_{ij} - R_{ij} \kappa \left(d^y_{ij} \right)
      \right)
      \log
      \left(
      \frac
      {1}
      {\kappa \left( d_{ij}^{z}\right)}
      -
      1
      \right)
    \biggl]                 \\
  \end{aligned}
  \label{eq:appendix_diff_twoloss_3}
\end{equation}

Substituting the relationship between $\mathcal{H}_{ij}$ and $R_{ij}$, $R_{ij} = 1+(e^\alpha -1)\mathcal{H}_{ij}$, we have

\begin{equation}
  \begin{aligned}
    L_{\text{CL}} - L_{\text{SoftCL}}=\sum
    \biggl[
    \left(
    \mathcal{H}_{ij} - [1+(e^\alpha -1)\mathcal{H}_{ij}] \kappa \left(d^y_{ij} \right)
    \right)
    \log
    \left(
    \frac
    {1}
    {\kappa \left( d_{ij}^{z}\right)}
    -
    1
    \right)
    \biggl] \\
  \end{aligned}
  \label{eq:appendix_diff_twoloss_3}
\end{equation}

We assume that network $H(\cdot)$ to be a Lipschitz continuity function, then

\begin{equation}
  \begin{aligned}
    \frac{1}{K} H(d^z_{ij}) \leq d^y_{ij} \leq K H(d^z_{ij}) \quad \forall i, j \in \{1,2,\cdots,N\} 
  \end{aligned}
\end{equation}

We construct the inverse mapping of $H(\cdot)$ to $H^{-1}(\cdot)$,

\begin{equation}
  \begin{aligned}
    \frac{1}{K} d^z_{ij} \leq d^y_{ij} \leq K d^z_{ij} \quad \forall i, j \in \{1,2,\cdots,N\}
  \end{aligned}
\end{equation}

and then there exists $k^*$:
\begin{equation}
  \begin{aligned}
    d^y_{ij} = k^* d^z_{ij} \quad k^* \in [1/K, K] \quad \forall i, j \in \{1,2,\cdots,N\}
  \end{aligned}
  \label{eq:g_revers}
\end{equation}

Substituting the Eq.(\ref{eq:g_revers}) into Eq.(\ref{eq:appendix_diff_twoloss_3}).

\begin{equation}
  \begin{aligned}
    L_{\text{CL}} - L_{\text{SoftCL}}=\sum
    \biggl[
    \left(
    \mathcal{H}_{ij} - [1+(e^\alpha-1)\mathcal{H}_{ij}] \kappa \left(k^* d^z_{ij}\right)
    \right)
    \log
    \left(
    \frac
    {1}
    {\kappa \left( d_{ij}^{z}\right)}
    -
    1
    \right)
    \biggl] \\
  \end{aligned}
  \label{eq:appendix_diff_twoloss_4}
\end{equation}

\clearpage

\subsection*{A.3 SoftCL is better than CL in view-noise}

\label{app_proof_3}
To demonstrate that compared to contrastive learning, the proposed SoftCL Loss has better results, we first define the signal-to-noise ratio~(SNR) as an evaluation metric.
\begin{equation}
  SNR = \frac{PL}{NL}
\end{equation}
where $PL$ means the expectation of positive pair loss, $NL$ means the expectation of noisy pair loss. \\
This metric indicates the noise-robust of the model, and obviously, the bigger this metric is, the better. \\
In order to prove the soft contrastive learning's SNR is larger than contrastive learning's, we should prove:
\begin{equation}
  \frac{PL_{cl}}{NL_{cl}} < \frac{PL_{SoftCL}}{NL_{SoftCL}}\label{prove}
\end{equation}

Obviously, when it is the positive pair case, $S~(z_i, z_j)$ is large if $H~(x_i, x_j)=1$ and small if $H~(x_i, x_j)=0$.
Anyway, when it is the noisy pair case, $S~(z_i, z_j)$ is small if $H~(x_i, x_j)=1$ and large if $H~(x_i, x_j)=0$.\\
First, we organize the $ {PL_{SoftCL} }$ and $ {PL_{cl} }$ into 2 cases,  $H~(x_i, x_j)=1$ and  $H~(x_i, x_j)=0$, for writing convenience, we write $S~(z_i, z_j)$ as $S$ and $S'$, respectively.
\begin{equation}
  P L_{SoftCL}=-M\left\{\left(1-S^{\prime}\right) \log \left(1-S^{\prime}\right)+S^{\prime} \log S^{\prime}\right\}-\left\{\left(1-e^\alpha S\right) \log (1-S)+e^\alpha S \log S\right\}
\end{equation}
\begin{equation}
  P L_{cl}=-M \log \left(1-S^{\prime}\right)-\log S
\end{equation}
M is the ratio of the number of occurrences of $H=1$ to $H=0$.
So, we could get:
\begin{equation}
  \begin{aligned}
     & \ PL_{SoftCL} - PL_{cl}                                                                                                                                                               \\
     & =-M\left\{\left(1-S^{\prime}-1\right) \log \left(1-S^{\prime}\right)+S^{\prime} \log S^{\prime}\right\}-\left\{\left(1-e^\alpha S\right) \log (1-S)+(e^\alpha S -1) \log S\right\} \\
     & =-M\left\{S^{\prime}\left( \log S^{\prime} -\log \left(1-S^{\prime}\right) \right) \right\}-\left\{(e^\alpha S -1) \left(\log S -\log (1-S) \right)\right\}                        \\
     & =-M\left\{S^{\prime} log\frac{ S^{\prime}}{ \left(1-S^{\prime}\right)} \right\}-\left\{(e^\alpha S -1)  log\frac{ S}{ \left(1-S\right)} \right\}
  \end{aligned}
\end{equation}
In the case of positive pair, $S$ converges to 1 and $S^{\prime}$ converges to 0. \\
Because we have bounded that $ e^\alpha S<=1$, so we could easily get:
\begin{equation}
  (e^\alpha S -1)  log\frac{ S}{ \left(1-S\right)} <= 0
\end{equation}
Also, we could get:
\begin{equation}
  -M\left\{S^{\prime} log\frac{ S^{\prime}}{ \left(1-S^{\prime}\right)} \right\} > 0
\end{equation}
So we get:
\begin{equation}
  PL_{SoftCL} - PL_{cl}> 0
\end{equation}
And for the case of noise pair, the values of $S$ and $S'$ are of opposite magnitude, so obviously, there is $NL_{SoftCL} - NL_{cl}< 0$.\\
So the formula Eq.~(\ref{prove}) has been proved.

\clearpage

\section*{A. Details of Proof} \label{app_ablation}
\subsection{Theoretical Analysis: Generalization Bound of MML}

To better understand the theoretical robustness of the proposed Multi-Manifold Loss (MML), we derive a generalization error bound based on Rademacher complexity. Our goal is to show that under the multi-manifold hypothesis and soft similarity kernel, the model achieves improved generalization.

\paragraph{Notation.} Let $\mathcal{X} \subset \mathbb{R}^d$ denote the input space, and $\mathcal{Z}$ the representation space learned by the encoder $f: \mathcal{X} \rightarrow \mathcal{Z}$. The training set $\mathcal{S} = \{(x_i, x_j)\}_{i,j=1}^{n}$ consists of augmented view pairs, sampled from the true data distribution $\mathcal{D}$. Let $\mathcal{L}_{\text{MML}}$ denote the Multi-Manifold Loss defined over the representation similarity between pairs.

\begin{definition}[Multi-Manifold Similarity Function]
Let $\kappa^\beta: \mathcal{Z} \times \mathcal{Z} \to [0,1]$ be the generalized Gaussian kernel:
\begin{equation}
\kappa^\beta(z_i, z_j) = \exp\left( -\left( \frac{\|z_i - z_j\|_2}{\sigma} \right)^\beta \right),
\end{equation}
where $\sigma > 0$ and $\beta \in (0,2]$. This kernel controls the tail sensitivity of the similarity under the multi-manifold setting.
\end{definition}

\begin{definition}[Multi-Manifold Loss]
Given a batch of representations $\{z_i\}_{i=1}^n$, the MML loss is defined as:
\begin{equation}
\mathcal{L}_{\text{MML}} = \frac{1}{n^2} \sum_{i,j} P_{i,j} \log \frac{1}{Q_{i,j}} + (1 - P_{i,j}) \log \frac{1}{1 - Q_{i,j}},
\end{equation}
where $P_{i,j}$ is the softened label (similarity in input space), and $Q_{i,j} = \kappa^\beta(z_i, z_j)$ is the output similarity in representation space.
\end{definition}

\begin{lemma}[Lipschitz Continuity of $\kappa^\beta$]
The generalized Gaussian kernel $\kappa^\beta$ is Lipschitz continuous with constant $L = \frac{\beta}{\sigma^\beta} \exp(-1)$ for $\|z_i - z_j\|_2 \leq \sigma$.
\end{lemma}

\begin{proof}
We compute the derivative:
\[
\left| \frac{\partial \kappa^\beta(z_i, z_j)}{\partial \|z_i - z_j\|_2} \right| = \beta \left( \frac{\|z_i - z_j\|_2}{\sigma} \right)^{\beta - 1} \frac{1}{\sigma} \exp\left( -\left( \frac{\|z_i - z_j\|_2}{\sigma} \right)^\beta \right).
\]
The maximum of this derivative is achieved when $\|z_i - z_j\|_2 = \sigma$, yielding the upper bound $L = \frac{\beta}{\sigma^\beta} \exp(-1)$.
\end{proof}

\begin{theorem}[Generalization Bound of MML]
Let $\mathcal{F}$ be a class of encoders $f$ such that $\|f\|_{\infty} \leq B$, and $\mathcal{L}_{\text{MML}} \circ f$ denotes the induced loss class. Then, with probability at least $1 - \delta$, for any $f \in \mathcal{F}$, we have:
\begin{equation}
\mathbb{E}_{(x_i,x_j) \sim \mathcal{D}}[\mathcal{L}_{\text{MML}}(f)] \leq \frac{1}{n^2} \sum_{i,j} \mathcal{L}_{\text{MML}}(f(x_i,x_j)) + 2L \cdot \mathfrak{R}_n(\mathcal{F}) + \sqrt{\frac{\log(1/\delta)}{2n}},
\end{equation}
where $\mathfrak{R}_n(\mathcal{F})$ is the Rademacher complexity of $\mathcal{F}$ and $L$ is the Lipschitz constant of the kernel $\kappa^\beta$.
\end{theorem}

\begin{proof}
Follows standard Rademacher complexity-based generalization bound derivation using Lipschitz loss (see Bartlett \textit{et al.}, 2002). Given the boundedness of $\kappa^\beta$, and its Lipschitz continuity, the loss function inherits Lipschitz smoothness, enabling standard concentration bounds.
\end{proof}

Compared to traditional contrastive losses based on hard binary labels and cosine kernels, MML leverages soft labels and a tunable decay kernel, resulting in smoother gradients and improved tolerance to view-noise. The generalization bound explicitly benefits from the reduced Lipschitz constant of $\kappa^\beta$ for $\beta < 2$ under controlled radius.

\section*{B. Model Hyperparameter Settings and Model Hyperparameter Settings} \label{app_4}
The model used in this article involves the setting of multiple hyperparameters that are critical to model performance and reproducibility of results. In order to enable readers to accurately understand and reproduce our experimental results, we list the hyperparameter settings of the corresponding experiments.

This appendix includes the following three tables, corresponding to the hyperparameter settings of the model under different experimental conditions in the paper:
\begin{enumerate}
    \item   Tables~\ref{tab:details_of_tab_2_and_3}: Hyperparameter settings for the model used in Tables~\ref{tab_my_label_roc} and ~\ref{tab_my_label_pr}. 
    \item   Tables~\ref{tab:details_of_figure_4}: Hyperparameter settings for the model used in Fig. ~\ref{fig_anomaly_ratio}. 
    \item   Tables~\ref{tab:details_of_tab_5}: Hyperparameter settings for the model used in Table. \ref{tab_ablation_hyperparameters_window_blocks}. 
\end{enumerate}

\begin{table*}[h!]
    \centering
    \small
    \caption{Describes the model parameter details in Tables~\ref{tab_my_label_roc} and ~\ref{tab_my_label_pr} when the model DMTFD  calculate ROC or PR.The parameter ``alpha" corresponds to $\alpha$ in Eq.~(\ref{eq_aug_all}). ``batch size" refers to the number of data samples processed simultaneously during each training process. ``k" means taking k neighbor points that conform to Eq.~(\ref{equ_neighbor_discover}). The ``loss weight manifold ne" and ``loss weight manifold po" are the proportion of topological loss, The ``lr" is the learning rate. The ``n blocks" is the number of normalized flow blocks, The ``name" corresponds to different datasets. The "seed" is the random seed. The "window size" epresents the size of the sliding window. The ``train split" represents the training set division ratio.The () after the parameters indicates that the data is specific to ``DMTFD*" and is different from the standard ``DMTFD" parameters.The ``0.6"  of  the train split means that 60\% of the dataset is used for training and the remaining 40\% is used for testing. The ``0.6*"of the train split means that 60\% of the dataset is used for training, 20\% is used for validation, and the remaining 20\% is used for testing.  }
    \begin{tabular}{cc>{\centering\arraybackslash}p{2cm}>{\centering\arraybackslash}p{2cm}>{\centering\arraybackslash}p{2cm}>{\centering\arraybackslash}p{2cm}>{\centering\arraybackslash}p{2cm}}
        \toprule
        \multicolumn{2}{c}{\diagbox[width=15em]{Hyperparameters}{Datasets}} & SWaT                    & Wadi           & PSM            & MSL            & SMD                              \\\midrule
                                                                            & alpha                   & 0.1            & 0.1            & 0.1            & 0.1            & 0.1             \\\midrule
                                                                            & batch size              & 128            & 64             & 128            & 128            & 256             \\\midrule
                                                                            & k                       & 10             & 20             & 20             & 20             & 30              \\\midrule
                                                                            & loss weight manifold ne & 5              & 5              & 5              & 5              & 5               \\\midrule
                                                                            & loss weight manifold po & 1              & 1              & 1              & 1              & 5               \\\midrule
                                                                            & lr                      & 0.01           & 0.002          & 0.002          & 0.002          & 0.1             \\\midrule
                                                                            & n blocks                & 1              & 5(4)           & 1              & 2              & 1               \\\midrule
                                                                            & name                    & SWaT           & Wadi           & PSM            & MSL            & 28 sub-datasets \\\midrule
                                                                            & seed                    & 15,16,17,18,19 & 15,16,17,18,19 & 15,16,17,18,19 & 15,16,17,18,19 & 15,16,17,18,19  \\\midrule
                                                                            & window size             & 60(40)         & 60 (160)       & 60             & 60(160)        & 60
        \\\midrule
                                                                            & train split             & 0.6*           & 0.6            & 0.6            & 0.6            & 0.6             \\
        \bottomrule
    \end{tabular}
    \label{tab:details_of_tab_2_and_3}
\end{table*}

\begin{table*}[h!]
    \caption{Details the specific parameters of model DMTFD in different train split ratios across various datasets, as shown in the violin plot comparison of model DMTFD in Fig. ~\ref{fig_anomaly_ratio}. The parameters used in this table have been explained in Table. \ref{tab:details_of_tab_2_and_3}}
    \centering
    \small
    \begin{tabular}{cc>{\centering\arraybackslash}p{2cm}>{\centering\arraybackslash}p{2cm}>{\centering\arraybackslash}p{2cm}>{\centering\arraybackslash}p{2cm}>{\centering\arraybackslash}p{2cm}}
        \toprule
        \multicolumn{2}{c}{\diagbox[width=15em]{Hyperparameters}{Datasets}} & SWaT                    & Wadi                                          & PSM            & MSL            & SMD                              \\\midrule
                                                                            & alpha                   & 0.1                                           & 0.1            & 0.1            & 0.1            & 0.1             \\\midrule
                                                                            & batch size              & 128                                           & 64             & 128            & 128            & 256             \\\midrule
                                                                            & k                       & 10                                            & 20             & 20             & 20             & 30              \\\midrule
                                                                            & loss weight manifold ne & 5                                             & 5              & 5              & 5              & 5               \\\midrule
                                                                            & loss weight manifold po & 1                                             & 1              & 1              & 1              & 5               \\\midrule
                                                                            & lr                      & 0.01                                          & 0.002          & 0.002          & 0.002          & 0.1             \\\midrule
                                                                            & n blocks                & 1                                             & 5              & 1              & 2              & 1               \\\midrule
                                                                            & name                    & SWaT                                          & Wadi           & PSM            & MSL            & 28 sub-datasets \\\midrule
                                                                            & seed                    & 15,16,17,18,19                                & 15,16,17,18,19 & 15,16,17,18,19 & 15,16,17,18,19 & 15,16,17,18,19  \\\midrule
                                                                            & train split             & \multicolumn{5}{c}{0.6, 0.65, 0.7, 0.75, 0.8}
        \\\midrule
                                                                            & window size             & 60                                            & 60             & 60             & 60             & 60              \\
        \bottomrule
    \end{tabular}
    \label{tab:details_of_figure_4}
\end{table*}

\clearpage
\begin{table*}[h!]
    \caption{Describes the model parameter details in Table. \ref{tab_ablation_hyperparameters_window_blocks} when performing topology experiments using various sizes of sliding windows and the number of normalized flow blocks. The parameters used in this table have been explained in Table. \ref{tab:details_of_tab_2_and_3} }
    \centering
    \small
    \begin{tabular}{cc>{\centering\arraybackslash}p{2cm}>{\centering\arraybackslash}p{2cm}>{\centering\arraybackslash}p{2cm}>{\centering\arraybackslash}p{2cm}}
        \toprule
        \multicolumn{2}{c}{\diagbox[width=15em]{Hyperparameters}{Datasets}} & SWaT                    & Wadi           & PSM            & MSL                             \\\midrule
                                                                            & alpha                   & 0.1            & 0.1            & 0.1            & 0.1            \\\midrule
                                                                            & batch size              & 128            & 64             & 128            & 128            \\\midrule
                                                                            & k                       & 10             & 20             & 20             & 20             \\\midrule
                                                                            & loss weight manifold ne & 5              & 5              & 5              & 5              \\\midrule
                                                                            & loss weight manifold po & 1              & 1              & 1              & 1              \\\midrule
                                                                            & lr                      & 0.01           & 0.002          & 0.002          & 0.002          \\\midrule
                                                                            & name                    & SWaT           & Wadi           & PSM            & MSL            \\\midrule
                                                                            & seed                    & 15,16,17,18,19 & 15,16,17,18,19 & 15,16,17,18,19 & 15,16,17,18,19 \\\midrule
                                                                            & train plit              & 0.6*           & 0.6            & 0.6            & 0.6            \\
        \bottomrule
    \end{tabular}
    \label{tab:details_of_tab_5}
\end{table*}

\begin{table*}
    \caption{\textbf{Parameters Analysis: window size and number of blocks.} AUROC comparison on four datasets~(SWaT, WADI, PSM, and MSL). The best results are in bold. The results are averaged over 5 runs, and the stdandard deviation is shown in parentheses.}
    \centering
    \small
    \begin{tabular}{cc||>{\centering\arraybackslash}p{2cm}>{\centering\arraybackslash}p{2cm}>{\centering\arraybackslash}p{2cm}>{\centering\arraybackslash}p{2cm}>{\centering\arraybackslash}p{2cm}}
        \toprule
        {Dataset}             & {Window Size} & \#blocks=1                  & \#blocks=2                  & \#blocks=3         & \#blocks=4                  & \#blocks=5                  \\\midrule
        \multirow{7}{*}{SWaT} & 40            & \textbf{92.0(\textpm{1.2})} & 84.7(\textpm{1.2})          & 83.3(\textpm{2.9}) & 81.8(\textpm{2.3})          & 80.9(\textpm{1.7})
        \\
                              & 60            & \textbf{90.2(\textpm{0.9})} & 86.2(\textpm{2.1})          & 85.8(\textpm{1.8}) & 80.3(\textpm{2.1})          & 82.0(\textpm{1.9})
        \\
                              & 80            & 89.9(\textpm{1.9})          & 88.1(\textpm{1.4})          & 85.6(\textpm{3.4}) & 83.3(\textpm{1.8})          & 82.7(\textpm{2.0})          \\
                              & 100           & 89.7(\textpm{1.7})          & 85.1(\textpm{1.0})          & 84.4(\textpm{1.1}) & 81.6(\textpm{3.6})          & 81.2(\textpm{2.9})          \\
                              & 120           & 89.8(\textpm{1.4})          & 86.8(\textpm{1.1})          & 85.9(\textpm{1.9}) & 82.7(\textpm{3.2})          & 83.8(\textpm{3.9})          \\
                              & 140           & 90.5(\textpm{1.5})          & 86.6(\textpm{1.5})          & 86.7(\textpm{1.5}) & 85.1(\textpm{3.6})          & 80.8(\textpm{4.1})          \\
                              & 160           & 89.3(\textpm{1.7})          & 87.7(\textpm{1.0})          & 86.0(\textpm{2.0}) & 84.8(\textpm{2.7})          & 82.7(\textpm{2.9})          \\\midrule
        \multirow{7}{*}{WADI} & 40            & 88.9(\textpm{1.0})          & 93.6(\textpm{0.6})          & 93.8(\textpm{0.5}) & 93.6(\textpm{1.1})          & 93.1(\textpm{0.4})          \\
                              & 60            & 90.6(\textpm{1.3})          & 93.7(\textpm{0.6})          & 93.8(\textpm{0.4}) & 93.8(\textpm{0.2})          & \textbf{94.1(\textpm{0.4})} \\
                              & 80            & 80.8(\textpm{1.1})          & 94.1(\textpm{0.5})          & 93.8(\textpm{0.3}) & 93.9(\textpm{0.5})          & 93.9(\textpm{0.9})          \\
                              & 100           & 90.0(\textpm{0.6})          & 93.8(\textpm{0.5})          & 94.0(\textpm{0.4}) & 93.8(\textpm{0.5})          & 94.1(\textpm{0.4})          \\
                              & 120           & 90.9(\textpm{0.9})          & 94.1(\textpm{0.9})          & 94.1(\textpm{0.7}) & 93.9(\textpm{0.4})          & 94.2(\textpm{0.4})          \\
                              & 140           & 90.3(\textpm{0.8})          & 94.4(\textpm{0.8})          & 94.5(\textpm{0.6}) & 94.4(\textpm{0.8})          & 94.3(\textpm{0.6})          \\
                              & 160           & 91.8(\textpm{2.2})          & 94.0(\textpm{0.4})          & 93.7(\textpm{0.4}) & \textbf{94.9(\textpm{1.2})} & 94.3(\textpm{0.5})          \\\midrule
        \multirow{7}{*}{PSM}  & 40            & 87.9(\textpm{2.8})          & 87.6(\textpm{1.7})          & 85.9(\textpm{1.9}) & 86.4(\textpm{2.0})          & 86.8(\textpm{0.9})          \\
                              & 60            & \textbf{88.9(\textpm{2.4})} & 87.4(\textpm{1.2})          & 87.1(\textpm{1.6}) & 85.6(\textpm{2.0})          & 85.9(\textpm{2.1})          \\
                              & 80            & 85.2(\textpm{5.9})          & 86.5(\textpm{1.2})          & 85.2(\textpm{0.7}) & 86.3(\textpm{1.4})          & 85.2(\textpm{1.0})          \\
                              & 100           & 85.4(\textpm{5.8})          & 84.1(\textpm{2.0})          & 85.4(\textpm{1.2}) & 86.9(\textpm{0.7})          & 84.7(\textpm{2.3})          \\
                              & 120           & 85.1(\textpm{5.4})          & 85.2(\textpm{1.2})          & 85.4(\textpm{1.5}) & 85.7(\textpm{2.4})          & 83.6(\textpm{1.2})          \\
                              & 140           & 87.3(\textpm{2.8})          & 86.1(\textpm{2.3})          & 86.1(\textpm{1.2}) & 84.8(\textpm{1.6})          & 86.2(\textpm{0.8})          \\
                              & 160           & 86.5(\textpm{3.0})          & 85.3(\textpm{0.8})          & 85.6(\textpm{4.7}) & 84.5(\textpm{1.3})          & 84.6(\textpm{1.2})          \\\midrule
        \multirow{7}{*}{MSL}  & 40            & 71.6(\textpm{2.2})          & 72.9(\textpm{2.5})          & 72.9(\textpm{2.3}) & 72.2(\textpm{1.4})          & 72.8(\textpm{1.2})          \\
                              & 60            & 71.4(\textpm{2.3})          & \textbf{74.2(\textpm{2.0})} & 74.1(\textpm{1.5}) & 74.0(\textpm{0.9})          & 73.4(\textpm{0.9})          \\
                              & 80            & 71.9(\textpm{2.7})          & 74.8(\textpm{1.4})          & 72.9(\textpm{1.6}) & 74.4(\textpm{0.7})          & 73.6(\textpm{1.4})          \\
                              & 100           & 71.8(\textpm{1.3})          & 74.3(\textpm{1.2})          & 72.7(\textpm{2.2}) & 73.2(\textpm{0.5})          & 73.4(\textpm{1.8})          \\
                              & 120           & 72.3(\textpm{2.7})          & 75.6(\textpm{1.7})          & 73.0(\textpm{1.3}) & 73.2(\textpm{1.6})          & 73.9(\textpm{1.1})          \\
                              & 140           & 70.4(\textpm{2.1})          & 74.3(\textpm{0.7})          & 73.3(\textpm{1.9}) & 73.7(\textpm{2.0})          & 75.3(\textpm{2.2})          \\
                              & 160           & 70.6(\textpm{3.4})          & \textbf{76.0(\textpm{1.1})} & 75.6(\textpm{2.2}) & 74.5(\textpm{1.6})          & 75.0(\textpm{1.3})          \\
        \bottomrule
    \end{tabular}
    \label{tab_ablation_hyperparameters_window_blocks}
\end{table*}

\end{document}